\documentclass[runningheads]{llncs}

\usepackage[T1]{fontenc}
\usepackage{graphicx}

\usepackage{multirow}   % Needed for \multirow
\usepackage{array}      % Better column formatting
\usepackage{booktabs} 
\usepackage{amsmath}
\usepackage{amssymb}
\usepackage{tikz}
\usepackage{algorithm}
\usepackage{algpseudocode}
\usepackage{graphicx}
\usepackage{subcaption}
\usepackage{xcolor}
\usepackage{wrapfig}
\usepackage{tcolorbox}

\newcommand{\step}[1]{\tikz[baseline=(char.base)]{
    \node[shape=circle,draw,inner sep=1pt] (char) {#1};}}

\begin{document}

\title{IntraShuffler: A Privacy Preserving Framework for Heterogeneous DP Federated Learning}

\titlerunning{IntraShuffler}

% \author{Anonymous Author(s)}

% \authorrunning{Anonymous}

% \institute{Anonymous Institution}
% \author{Anonymous Author(s)}
% \institute{Anonymous Institution(s)}

\author{Farhin Farhad Riya\inst{1}
\and
Olivera Kotevska\inst{2}
\and
Jinyuan Stella Sun\inst{1}
}
\authorrunning{F. Riya et al.}
\institute{University of Tennessee, Knoxville, USA \and
Oak Ridge National Laboratory, USA \\
\email{friya@vols.utk.edu}; \email{kotevskao@ornl.gov}; \email{jysun@utk.edu}
}

\maketitle         
\pagestyle{plain}

\begin{abstract}

Heterogeneous Differential Privacy (HDP) in Federated Learning (FL) allows clients to select individual privacy budgets ($\varepsilon_i$) according to institutional policies and data sensitivity. In practice, many HDP-FL systems employ $\varepsilon$-aware server aggregation to improve model utility by re-weighting client updates according to their declared privacy budgets. However, gradient updates in FL retain structural patterns induced by non-independent and identically-distributed (non-IID) data, and these additional signals exposed by $\varepsilon$-aware aggregation create new opportunities for inference by an honest-but-curious server. In this work, we first show that a server equipped with gradient denoising and surrogate modeling can mount a \emph{Privacy Inference Attack} that infers distributional attributes of clients and links updates from the same client across training rounds,  measured via surrogate inference accuracy and linkage success, under realistic knowledge constraints.

The Shuffle-Model has been widely studied as a defense against such inference risks by anonymizing update sources, but it is fundamentally incompatible with HDP-FL $\varepsilon$-aware aggregation. To address this challenge, we propose \textbf{IntraShuffler}, a middleware defense framework designed for HDP-FL systems. IntraShuffler introduces a privacy-aware shuffling mechanism that groups clients into privacy-compatible buckets and performs parameter-level shuffling within each bucket to disrupt persistent gradient structure while preserving $\varepsilon$-aware aggregation. Experiments across four different datasets show that IntraShuffler reduces gradient recoverability by over \textbf{60\%} and decreases surrogate inference accuracy from \textbf{0.78} to \textbf{0.33} while maintaining comparable model utility across multiple FL aggregation rules.

% \keywords{Federated Learning  \and Heterogeneous Differential Privacy \and Shuffler \and Privacy Leakage \and Privacy Defense \and Bucketing}
\end{abstract}

% \vspace{-0.3em}
\section{Introduction}
\vspace{-0.7em}

Federated Learning (FL) enables collaborative model training across distributed clients without sharing raw data, where clients compute local updates that are aggregated by a central server \cite{zhang2021survey,mammen2021federated}. Despite this decentralized design, prior work shows that shared updates can still leak sensitive information about client data and behavior \cite{melis2019exploiting,geiping2020inverting}. Differential Privacy (DP) is widely adopted to mitigate such risks by adding noise to client updates \cite{abadi2016deep,wei2020federated}. In practice, clients often have heterogeneous privacy requirements, motivating Heterogeneous Differential Privacy (HDP), where each client selects an individual privacy budget $\varepsilon_i$ \cite{HDP-FL-li2023,ling2024efficient}. To ensure convergence under heterogeneous noise, many HDP-FL systems adopt $\varepsilon$-aware aggregation that re-weights updates based on their declared privacy budgets. However, privacy risks persist as gradient updates retain structural patterns induced by non-IID data, and $\varepsilon$-aware aggregation introduces additional signal that can be exploited to enable inference. This raises an important question: \emph{Can an honest-but-curious server infer client participation or distributional attributes with these signals even when LDP is applied?} For example, a client with a large privacy budget ($\varepsilon=8$) produces consistently low-noise updates with similar directionality across rounds, enabling linkage even under shuffling. We identify that the root cause of privacy leakage in HDP-FL is not insufficient noise or lack of anonymization, but the persistence of client-specific gradient structure under $\varepsilon$-aware aggregation.

In this work, we demonstrate that such inference is indeed possible in HDP-FL systems. Leveraging two well-established analysis capabilities, gradient denoising techniques that partially recover gradient directionality and surrogate modeling that captures client data characteristics, we develop a unified \emph{Privacy Inference Attack} framework that enables the server to infer the update source and data distributional attributes under realistic knowledge constraints.

A natural defense against such attacks is the shuffle-model, where an intermediary shuffler anonymizes the origin of client updates before they reach the server \cite{cheu2019distributed,erlingsson2019amplification,bittau2017prochlo}. Shuffle-based defenses reduce linkability by permuting messages and have been widely studied in homogeneous privacy settings. However, these approaches are not directly compatible with HDP-FL systems that rely on $\varepsilon$-aware server aggregation. Because $\varepsilon$-aware aggregation requires associating each update with its privacy tier, message-level anonymization cannot remove identity-linked signals without also removing the weights required for correct aggregation. As a result, existing shuffle-based defenses are insufficient to prevent inference attacks in HDP-FL settings.

To address this challenge, we propose \textbf{IntraShuffler}, a lightweight middleware defense designed for HDP-FL systems. Rather than eliminating $\varepsilon$-aware aggregation, IntraShuffler preserves it while disrupting the structural signals that enable inference. The key insight is that non-IID data induces persistent structural patterns in gradient updates. When updates can be statistically linked across rounds, these patterns enable inference about client participation and data characteristics. IntraShuffler mitigates this risk by strategically grouping clients into privacy-compatible buckets and performing parameter-level shuffling within each bucket. The bucketing mechanism ensures sufficient anonymity under dynamic client participation while preserving compatibility with $\varepsilon$-aware aggregation, while the intra-bucket parameter-level shuffling disrupts recurring gradient structures that enable cross-round linkage.

While several other defenses have been proposed in the literature to mitigate inference risks in FL, such as Secure Aggregation, Homomorphic Encryption, Secure Multi-party Computation, etc., many require modifying either client training procedures or server aggregation mechanisms or both. In contrast, shuffle-based architectures introduce an anonymization layer between clients and the server without altering local training objectives or aggregation protocols, making them attractive for deployment. This motivated us to propose a defense that follows a similar middleware architecture and can be integrated into existing FL systems without modifying client training pipelines or server optimizers. Consequently, it remains compatible with widely used federated optimizers such as \textsc{FedAvg}, \textsc{FedOpt}, and \textsc{FedProx}. Through extensive experiments on both vision and time-series workloads, we demonstrate that IntraShuffler significantly reduces gradient-based inference risks while maintaining model utility.

Again, removing $\varepsilon$-aware aggregation from the FL setting is not a practical solution, as such mechanisms are commonly adopted to maintain convergence and model utility under HDP (e.g., Projection-based Aggregation methods ~\cite{osti_10332814}, Weighted Aggregation methods \cite{ling2024efficient}, Client Sampling methods \cite{ma2025power,xu2026optimal}). Our goal is therefore not to restrict the server's access to aggregation signals, but to prevent the server from exploiting these signals to infer sensitive properties of participating clients.

% \vspace{+0.5em}
The main contributions of this paper can be summarized as follows:
\vspace{-0.5em}
\begin{itemize}
\item We introduce a new \emph{Privacy Inference Attack} that enables an honest-but-curious server to infer client participation and data distributional attributes in HDP-FL using gradient denoising and surrogate modeling.
\item We identify that privacy leakage in HDP-FL stems from persistent client-specific gradient structure under $\varepsilon$-aware aggregation, which enables cross-round linkage and inference despite local DP and shuffling.

\item We propose \textbf{IntraShuffler}, a privacy-aware bucketing and parameter-shuffling defense, to disrupt persistent gradient structure while preserving compatibility with $\varepsilon$-aware aggregation.

\item We evaluate IntraShuffler across vision and time-series tasks, showing over 60\% reduction in gradient recoverability, a drop in surrogate inference accuracy from 0.78 to 0.33, and near-random cross-round source inference, while maintaining comparable model utility. We further report bucket-level attack success rates, revealing higher vulnerability for high-$\varepsilon$ clients and demonstrating that IntraShuffler mitigates this imbalance.
% \item We empirically evaluate IntraShuffler across vision and time-series tasks, showing over 60\% reduction in gradient recoverability, a drop in surrogate inference accuracy from 0.78 to 0.33, and near-random cross-round source inference, while maintaining comparable model utility. We further report bucket-level attack success rates that reveal a privacy imbalance in HDP-FL systems, where higher-$\varepsilon$ clients are more vulnerable, and show that IntraShuffler substantially mitigates this leakage across privacy buckets.

% \item \textbf{Key Insight.}
% We identify that privacy leakage in HDP-FL stems from persistent client-specific gradient structure under $\varepsilon$-aware aggregation, which enables cross-round linkage and inference despite local DP and shuffling.

% \item \textbf{Method.}
% We propose \textbf{IntraShuffler}, a privacy-aware bucketing and parameter-level shuffling framework that disrupts structural signals while preserving compatibility with $\varepsilon$-aware aggregation.

% \item \textbf{Empirical Validation.}
% We evaluate IntraShuffler across vision and time-series tasks, showing over 60\% reduction in gradient recoverability, a drop in surrogate inference accuracy from 0.78 to 0.33, and near-random cross-round source inference, while maintaining comparable model utility. We further report bucket-level attack success rates, revealing higher vulnerability for high-$\varepsilon$ clients and demonstrating that IntraShuffler mitigates this imbalance.

\end{itemize}

\vspace{-1.0em}
\section{Related Work}
\vspace{-0.5em}

\noindent\textbf{Inference Attacks in Federated Learning.}
Although gradients were initially shown to enable reconstruction attacks such as Deep Leakage from Gradients and its variants~\cite{zhu2019deep,geiping2020inverting}, subsequent work demonstrated that gradient updates can also leak higher-level information about participating clients. Prior studies have shown that adversaries can infer attributes of client data or training membership from shared updates. For example, Melis et al.~\cite{melis2019exploiting} showed that gradient updates may reveal client-level attributes and behavioral properties in collaborative learning systems, while Nasr et al.~\cite{nasr2019comprehensive} demonstrated that FL remains vulnerable to membership inference attacks (MIAs) even when secure aggregation is applied. More recently, source inference attacks (SIAs) have been proposed to identify the client responsible for a particular update by exploiting persistent gradient patterns and distributional differences across clients~\cite{hu2021source,hu2023source}. These works highlight that gradient-level signals can enable inference about participating clients even when raw data are never directly shared.

\noindent\textbf{Shuffle-Based Privacy Mechanisms.}
Shuffling improves privacy in distributed systems by anonymizing the origin of
client messages. Its roots trace back to anonymity networks and cryptographic
mixing protocols~\cite{dingledine2004tor,hohenberger2014mixnets,corriangibbs2015riposte,lu2019shuffle,abraham2020mpcshuffle}, and it was later introduced to DP through Prochlo~\cite{bittau2017prochlo} and formalized in the shuffle-model literature~\cite{cheu2019distributed,erlingsson2019amplification,wang2025rafls}. In this model, clients locally perturb messages before a shuffler permutes them, providing anonymity and
privacy amplification. Shuffle-based designs have been used in practical
systems such as Google and Brave~\cite{bittau2017prochlo} and explored in FL to
reduce update linkability or strengthen DP guarantees~\cite{sun2020secure,lebun2022shuffle,yang2023shufflefl,liu2021flame}. However, most shuffle-based FL defenses assume homogeneous DP~\cite{abadi2016deep}, whereas practical FL
deployments often involve HDP with client-specific privacy budgets and
$\varepsilon$-aware aggregation~\cite{jorgensen2015conservative,lin2023heterogeneous}. Recent HDP shuffle studies focus mainly on privacy accounting~\cite{balcer2022shuffle,liu2023echo}, leaving the empirical inference-attack risk in HDP-FL less explored.

He et al.~\cite{he2023clustered} introduce a form of parameter-level shuffling in FL, where each client locally shuffles model weights by attaching identifiers and locations before transmission. However, this mechanism does not follow the shuffle-model architecture; the server later reconstructs the model by regrouping the labeled weights, leaving the underlying gradient structure intact.

To the best of our knowledge, the recent ICLR’26 work of Athanasiou et al.~\cite{Athanasiou_2024} is the first to study parameter-level shuffling in the shuffle-model for FL. They demonstrate that such shuffling can remain vulnerable to SIA via recombination of final-layer parameters, under strong assumptions including prior knowledge of the target client’s data distribution and no client-side local-DP (LDP). The authors note that the attack is not viable without this knowledge. In contrast, we consider client-side DP guarantee with a stricter and more practical threat model and show that under standard model architectures, such recombination attacks become significantly infeasible.

% Our work, therefore, investigates inference risks in HDP-FL systems employing $\varepsilon$-aware aggregation and introduces \textbf{IntraShuffler}, a parameter-level shuffle-based defense designed to disrupt persistent gradient structures while preserving compatibility with HDP settings and standard FL aggregation protocols.

\vspace{-0.5em}
\section{Preliminaries}
\label{sec:prelim}
\vspace{-0.5em}

\noindent\textbf{Notation.}
Let $[m]:= \{1,\dots,m\}$ denote the set of clients participating in a federated learning (FL) system. For a vector $x \in \mathbb{R}^d$, $x[j]$ denotes its $j$-th coordinate and $\|x\|_2$ denotes the Euclidean norm.

\vspace{+0.5em}
\noindent\textbf{Federated Learning with Heterogeneous Differential Privacy.} We consider a standard FL setting with $m$ clients and a central server. At each communication round $t \in [T]$, a subset of clients $\mathcal{C}_t \subseteq [m]$ participates in training. Each client $i \in \mathcal{C}_t$ holds a local dataset $D_i$ drawn from a client-specific distribution $\mathcal{P}_i$. Given the current global model parameters $w^{(t)}$, client $i$ computes a local update (e.g., gradient or model delta)
\(\displaystyle
g_i^{(t)} \in \mathbb{R}^d .
\)To protect client data, the update is perturbed locally before transmission to the server. In HDP, each client selects an individual privacy budget $\varepsilon_i$, resulting in client-specific noise levels. A common instantiation uses a Gaussian perturbation,
\begin{equation}
\tilde g_i^{(t)} = g_i^{(t)} + \xi_i^{(t)}, 
\qquad
\xi_i^{(t)} \sim \mathcal{N}(0,\sigma_i^2 I_d)
\label{eq:hpd-noise}
\end{equation}
where gradients are first clipped to a norm bound $C$ and Gaussian noise is added, with $\sigma_i$ calibrated to satisfy $(\varepsilon_i,\delta)$-DP under a standard DP-SGD mechanism~\cite{abadi2016deep}. The server is assumed to know the DP mechanism (including $C$ and $\delta$) and  receives the perturbed updates $\tilde g_i^{(t)}$ with its associated privacy level $\varepsilon_i$

\vspace{+0.5em}
\noindent\textbf{$\varepsilon$-Aware Aggregation in HDP-FL.} Because HDP introduces heterogeneous noise levels across clients, many HDP-FL systems employ \emph{$\varepsilon$-aware aggregation} to maintain model utility and training stability. In this setting, the server aggregates client updates using weights that depend on the corresponding privacy tiers or noise levels. Let $\alpha_i^{(t)}$ denote the aggregation weight assigned to client $i$ in round $t$. The global model update can be written as
\begin{equation}
w^{(t+1)} = w^{(t)} +
\sum_{i\in\mathcal{C}_t} \alpha_i^{(t)} \tilde g_i^{(t)}
\label{eq:eps-agg}
\end{equation}

The weights $\alpha_i^{(t)}$ depend on the client's declared privacy budget $\varepsilon_i$, from which the server can infer the corresponding noise scale $\sigma_i$ under the known DP mechanism (assigning lower weight to noisier updates)~\cite{ling2024efficient,li2022hdpfl}. In this work, we assume aggregation relies only on this privacy-level information.

\vspace{+0.5em}
\noindent\textbf{Shuffle-Inspired Middleware Architecture.}
To reduce linkability between client updates and their sources, some systems
place a \emph{shuffler} between clients and the server. The shuffler permutes or
transforms updates before aggregation. Formally, let
\(\tilde g_{1:n_t}^{(t)}=\{\tilde g_i^{(t)}:i\in\mathcal{C}_t\}\)
denote the perturbed updates in round $t$, where $n_t=|\mathcal{C}_t|$.
A classical shuffle mechanism applies a random permutation $\pi$,
\(
\mathsf{Shuf}(\tilde g_{1:n_t}^{(t)}) =
(\tilde g_{\pi(1)}^{(t)},\dots,\tilde g_{\pi(n_t)}^{(t)}),
\)
removing the direct client--message correspondence. However, in HDP-FL, the
server must retain each update's privacy tier for $\varepsilon$-aware
aggregation, making naive message-level shuffling potentially incompatible with
HDP aggregation. This motivates shuffle-inspired mechanisms that reduce
linkability while preserving aggregation-relevant privacy information.
% \noindent\textbf{Shuffle-Inspired Middleware Architecture.} To reduce the linkability between client updates and their originating clients, some systems introduce an intermediate \emph{shuffler} between clients and the server. The shuffler acts as a middleware component that transforms or permutes updates before they reach the server. Formally, let
% \(\displaystyle
% \tilde g_{1:n_t}^{(t)} =
% \{\tilde g_i^{(t)} : i \in \mathcal{C}_t\}
% \),
% denote the set of perturbed updates generated in round $t$, where $n_t = |\mathcal{C}_t|$. A classical shuffle mechanism applies a random permutation $\pi$ over these updates,
% \(\displaystyle
% \mathsf{Shuf}(\tilde g_{1:n_t}^{(t)}) =
% (\tilde g_{\pi(1)}^{(t)},\dots,\tilde g_{\pi(n_t)}^{(t)}),
% \), thereby removing the direct correspondence between clients and messages. In HDP-FL systems, however, the server must retain information about the privacy tier associated with each update in order to perform $\varepsilon$-aware aggregation. As a result, naive message-level shuffling may be incompatible with HDP aggregation rules. This motivates the design of shuffle-inspired mechanisms that transform updates while preserving the information required for aggregation.

\vspace{-0.5em}
\section{Threat Model}
\label{sec:threat}
\vspace{-0.6em}

\noindent\textbf{Adversary.}
We consider an \emph{honest-but-curious} server that follows the prescribed protocol, including $\varepsilon$-aware aggregation (Eq.~\ref{eq:eps-agg}), but attempts to infer client-specific information from observed updates.

\textbf{- Observations.}
At each communication round $t$, a subset of clients $\mathcal{C}_t$
participates and releases locally perturbed updates following
Eq.~(\ref{eq:hpd-noise}). After shuffling, the server observes an anonymized
update set $\mathcal{Y}_t$, and over training observes the sequence
$(\mathcal{Y}_1,\dots,\mathcal{Y}_T)$.
% At each communication round $t$, a subset of clients $\mathcal{C}_t$ participates and releases locally perturbed updates following Eq.~(\ref{eq:hpd-noise}). After passing through the Shuffler described in Section~\ref{sec:prelim}, the server observes an anonymized set of updates $\mathcal{Y}_t$. Across training rounds, the server observes the sequence $(\mathcal{Y}_1,\dots,\mathcal{Y}_T)$.

\textbf{- Knowledge.}
The adversary knows the training protocol, model architecture, and the privacy tier associated with each received update (e.g., $\varepsilon_i$ or the corresponding noise scale $\sigma_i$), as required for $\varepsilon$-aware aggregation. The server may also access auxiliary or historical data from similar domains for offline analysis, but does not access raw data, clean gradients, or shuffler randomness.

% \vspace{+0.5em}
\noindent\textbf{Trust Assumptions.}
Clients follow the protocol honestly and apply LDP with chosen privacy budgets. Clients send $(\tilde g_i^{(t)}, \varepsilon_i)$ to the Shuffler, which outputs a permuted multiset of updates with their associated tier identifiers to the server. The shuffler is treated as a trusted, non-colluding middleware component and can be instantiated as a single entity or via distributed shuffling (e.g., Prochlo \cite{bittau2017prochlo}). If the shuffler is malicious or colludes with the server, guarantees reduce to baseline LDP without additional anonymity (Appendix~\ref{app:malicious_shuffler}).

This work focuses on inference risks arising from the statistical structure of gradient updates under HDP. Active attacks such as poisoning, protocol deviation, side-channel attacks, or client collusion are outside the scope of this work.

\vspace{-0.7em}
\section{Privacy Inference Attack}
\label{sec:attack}
\vspace{-0.7em}
% \subsection{Attack Motivation}

Prior SIA~\cite{hu2021source,hu2023source,Athanasiou_2024} attacks often assume strong adversarial knowledge, such as access to client data distributions or exact training datasets. In contrast, our attack relies solely on protocol-level information and anonymized updates. We demonstrate that even under these limited assumptions, structural signals in gradient updates can still enable privacy inference despite LDP and shuffle anonymization.

\vspace{-0.7em}
\subsubsection{Structural Leakage in Gradient Updates}
\vspace{-0.7em}
Shuffle-model \cite{cheu2019distributed,erlingsson2019amplification,bittau2017prochlo} provides source anonymity by removing the direct mapping between updates and client identities. However, the gradient update itself may still contain structural signals reflecting the underlying client data distribution. Let $\mathcal{P}_i$ denote the data distribution of client $i$. Under standard smoothness assumptions, the expected update direction depends on this distribution,
\(\displaystyle
\mu_i^{(t)} := \mathbb{E}[g_i^{(t)}] \approx \nabla F(\theta^{(t)}; \mathcal{P}_i).
\) When client datasets are non-IID, these expected update directions vary across clients. Even after local perturbation, the observed update remains a noisy observation of this distribution-dependent structure. This structural signal forms the basis of the privacy inference attack.

\begin{figure} [t]
    \centering
    \includegraphics[width=0.7\columnwidth]{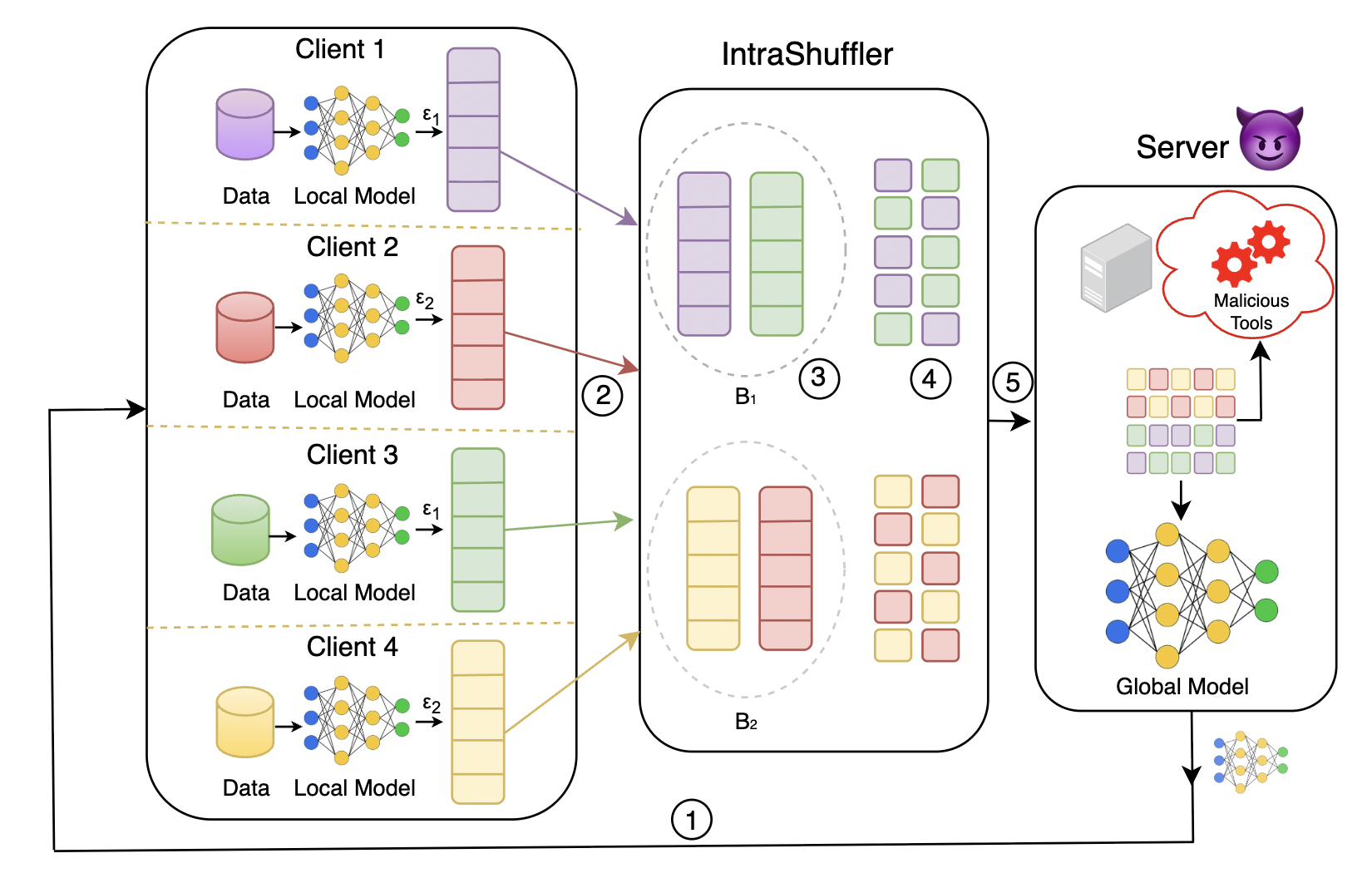}
    \vspace{-0.5em}
    \caption{Overview of the IntraShuffler Framework.}
    \label{fig:Overview}
\end{figure}

\vspace{-0.3em}
\subsubsection{Attack Methodology}
\vspace{-0.5em}
The attacker aims to infer coarse properties of the underlying client data distribution from the observed gradient structure and to link updates generated by the same client across training rounds.

\noindent\textbf{Stage 1. Gradient Denoising.}
Under HDP, clients perturb updates using noise levels determined by their privacy budgets. Because the server observes both the noisy update $\tilde g_i^{(t)}$ and the associated privacy budget $\varepsilon_i$, and because the training protocol specifies the clipping bound and noise calibration mechanism, the server can infer the corresponding noise scale. The attacker employs an $\varepsilon$-aware learned denoiser $f_\theta(\tilde g, \varepsilon)$ trained on synthetic pairs of clean and noisy gradients generated using the known DP mechanism. The model is trained to minimize reconstruction error (e.g., $\ell_2$ loss) and learns to map noisy updates to their denoised estimates $\hat g = f_\theta(\tilde g, \varepsilon)$, partially restoring distribution-dependent structural patterns.

The denoising step is not unique to HDP, since a server in a homogeneous-DP
setting could also train a fixed-noise denoiser if the common privacy budget is known. The HDP-specific risk is that denoising becomes privacy-tier adaptive. Different clients use different $\varepsilon_i$ values, and the server-visible privacy tier lets the attacker condition recovery on each update's noise scale. Consequently, high-$\varepsilon$ updates contain less noise and become more recoverable and linkable than low-$\varepsilon$ updates. Thus, the attack exploits the combination of heterogeneous noise scales, server-visible privacy tiers, and non-IID gradient structure, rather than denoising alone.

\noindent\textbf{Stage 2. Behavioral Inference via Surrogate Model.}
The attacker employs a surrogate model trained on auxiliary data from a similar domain (e.g., public datasets or historical updates). Behavioral labels correspond to coarse data attributes (e.g., class distribution or domain category), derived directly from the auxiliary dataset. The number of classes is task-dependent and aligned with the downstream prediction objective. The denoised updates $\hat g$ are used as input features, enabling the model to associate gradient structure with these coarse distributional properties.

\noindent\textbf{Stage 3. Cross-Round Linkage.}
When the same client participates across multiple rounds, the distribution-dependent update direction remains statistically correlated over time. The attacker measures similarity between denoised updates across rounds using cosine similarity,
\(
s(\hat g^{(t)}, \hat g^{(t')}) =
\frac{\hat g^{(t)} \cdot \hat g^{(t')}}{\|\hat g^{(t)}\|\,\|\hat g^{(t')}\|}.
\) Updates are linked by nearest-neighbor matching based on the similarity score across rounds, enabling the attacker to associate updates and track inferred behavioral properties over time. We consider a closed-world setting where the candidate set consists of updates within a fixed window of rounds, allowing multiple participations per client and linkage performance is evaluated using top-1 matching accuracy over the candidate set.

\vspace{+0.5em}
\noindent\textbf{Implication for Defense.}
The attack reveals that privacy leakage arises not from the absence of noise or shuffling, but from the persistence of client-dependent gradient structure. While shuffling removes explicit identifiers, it does not alter the internal geometry of updates, leaving distribution-dependent patterns detectable after denoising and across rounds. Effective defenses must therefore disrupt this structural consistency while remaining compatible with $\varepsilon$-aware aggregation, motivating the design of \textit{IntraShuffler}.

% While shuffling removes explicit client identifiers, it does not modify the internal geometry of gradient updates. As a result, distribution-dependent patterns may remain detectable after denoising and across training rounds. Effective defenses, therefore, must disrupt the structural consistency of gradient representations while remaining compatible with $\varepsilon$-aware aggregation. This observation motivates the design of the proposed \textit{IntraShuffler} mechanism.

\vspace{-0.5em}
\section{IntraShuffler Framework}
\label{sec:intrashuffler}
\vspace{-0.5em}

Motivated by the structural leakage identified in Section~\ref{sec:attack}, we propose \textbf{IntraShuffler}, a shuffle-based middleware mechanism that disrupts client-dependent gradient structure while remaining compatible with HDP-FL systems that employ $\varepsilon$-aware aggregation. The mechanism operates between clients and the server and does not modify the client-side training procedures or the local DP mechanisms applied by clients.

\vspace{0.4em}
\noindent\textbf{Design intuition.}
IntraShuffler reduces the ability of the server to reconstruct coherent per-client gradients by shuffling parameters across multiple clients before updates reach the server. The strength of this protection depends on the size of the mixing set: when too few clients are mixed together, parameter recombination becomes easier. This motivates enforcing a minimum bucket population to ensure sufficiently large anonymity sets for parameter-level shuffling.

\begin{definition}[Initial Privacy Buckets]
Let clients be indexed by $i \in [m]$, each declaring a privacy budget $\varepsilon_i$. At communication round $t$, the participating clients $\mathcal{C}_t$ are first grouped according to their declared privacy budgets \(
\mathcal{T}_\ell^{(t)} := \{ i \in \mathcal{C}_t : \varepsilon_i = \varepsilon^{(\ell)} \},
\) where $\{\varepsilon^{(\ell)}\}$ denotes the set of privacy budget values appearing among the participating clients in that round. Thus, all clients within an initial group share the same privacy budget and corresponding nominal local noise scale.
\end{definition}

These initial groups preserve the privacy budget information required by the server in HDP-FL systems. However, some groups may contain very few clients, which provides only limited protection under parameter-level shuffling.

\vspace{0.5em}
\noindent\textbf{Adaptive shuffling buckets and minimum population.}
Let $n_k := |\mathcal{B}_k|$ denote the size of a final shuffling bucket $\mathcal{B}_k$. IntraShuffler enforces a minimum bucket population $n_{\min}$, treated as a design hyperparameter controlling the minimum anonymity-set size required for effective shuffling. Buckets are ordered by privacy budget, and adjacency refers to neighboring buckets in this ordered sequence. If an initial group satisfies $|\mathcal{T}_\ell^{(t)}| \ge n_{\min}$, it is used directly; otherwise, it is merged with an adjacent group to form a larger bucket. The merge candidate is chosen to minimize privacy distance, preferring the smaller neighbor and breaking ties toward the lower-$\varepsilon$ group. The final buckets $\{\mathcal{B}_1,\dots,\mathcal{B}_K\}$ are obtained by iteratively merging underpopulated adjacent groups until $n_k \ge n_{\min}$. This adaptive merging introduces a privacy--utility trade-off where larger buckets improve anonymity but may mix updates with different noise levels, reducing the benefits of fine-grained privacy grouping.

\begin{definition}[Parameter-Level Shuffling]
Let $\tilde g_i^{(t)} \in \mathbb{R}^d$ denote the locally private update generated by client $i$ at round $t$. For a final shuffling bucket $\mathcal{B}_k$, define the set of privatized updates
\(
\mathcal{G}_k^{(t)} := \{ \tilde g_i^{(t)} : i \in \mathcal{B}_k \}.
\) IntraShuffler applies independent random permutations to parameter positions across clients within the bucket. For each parameter index $j \in \{1,\dots,d\}$, a permutation $\pi_{k,j}^{(t)}$ over the client indices in $\mathcal{B}_k$ is sampled and the shuffled updates are constructed as
\(
\hat g_i^{(t)}[j] = \tilde g_{\pi_{k,j}^{(t)}(i)}^{(t)}[j].
\) Thus, each output update $\hat g_i^{(t)}$ is no longer a coherent per-client gradient vector, but a parameter-wise mixture of values drawn from clients within the same shuffling bucket.
\end{definition}

\noindent Parameter-level shuffling is important because shuffling at coarser granularity (e.g., layer-wise) preserves intra-layer parameter correlations. Since inference attacks exploit these fine-grained structures, such coarse shuffling provides limited protection. In contrast, parameter-level shuffling disrupts these dependencies, effectively removing the signals required for reliable inference. For a model with layers $C, FC_1, FC_2$, layer-wise shuffling permutes whole layers across clients, \(\{C_{\pi_0(i)}, FC_{1,\pi_1(i)}, FC_{2,\pi_2(i)}\}_{i=1}^n\), preserving intra-layer grouping, whereas parameter-level shuffling permutes individual parameters.

\begin{algorithm}[t]
\caption{IntraShuffler with adaptive minimum-population shuffling buckets}
\label{alg:intrashuffler}
\scriptsize
\begin{algorithmic}[1]
\For{round $t = 1$ to $T$}
    \State Receive privatized updates $\{\tilde g_i^{(t)}\}_{i \in \mathcal{C}_t}$ and privacy budgets $\{\varepsilon_i\}_{i \in \mathcal{C}_t}$
    \State Form initial groups $\{\mathcal{T}_\ell^{(t)}\}$ by grouping clients with identical $\varepsilon$
    \State Initialize shuffling buckets from these groups
    \While{there exists a bucket $\mathcal{B}_k$ with $|\mathcal{B}_k| < n_{\min}$}
        \State Let $\mathcal{N}(\mathcal{B}_k)$ denote the adjacent buckets of $\mathcal{B}_k$
        \State Select merge candidate
        \(
        \mathcal{B}^* =
        \arg\min_{\mathcal{B}_j \in \mathcal{N}(\mathcal{B}_k)}
        \big(
        |\varepsilon_j - \varepsilon_k|,
        |\mathcal{B}_j|,
        \varepsilon_j
        \big)
        \)
        \State Merge $\mathcal{B}_k$ with $\mathcal{B}^*$
    \EndWhile
    \For{each final shuffling bucket $\mathcal{B}_k$}
        \For{each parameter index $j = 1,\dots,d$}
            \State Randomly permute $\{\tilde g_i^{(t)}[j] : i \in \mathcal{B}_k\}$
        \EndFor
    \EndFor
    \State Forward the parameter-shuffled updates together with the corresponding bucket-level privacy labels
\EndFor
\end{algorithmic}
\end{algorithm}

\vspace{+0.3em}

\noindent\textbf{Aggregation compatibility.}
After parameter-level shuffling within a bucket $\mathcal{B}_k$, the resulting updates are no longer coherent client-specific gradients. Accordingly, IntraShuffler associates shuffled outputs with a bucket-level privacy label for aggregation. Since all updates within a bucket share the same aggregation weight $\alpha_k$, parameter-level permutation preserves the bucket’s weighted sum. This allows $\varepsilon$-aware aggregation to proceed without modifying the server-side optimization rule, $w^{(t+1)} = w^{(t)} + \sum_{k=1}^{K} \alpha_k \sum_{i \in \mathcal{B}_k} \hat{g}_i^{(t)}$.

\begin{proposition}[Permutation Invariance of Parameter Multisets]
\label{prop:agg}
For any final shuffling bucket $\mathcal{B}_k$ and parameter index $j$, parameter-level shuffling preserves the multiset of parameter values 
\(
\{\tilde g_i^{(t)}[j] : i \in \mathcal{B}_k\}.
\)
\end{proposition}

\noindent This follows since shuffling applies a permutation over the parameter values, which does not alter the underlying multiset.

\begin{proposition}[Recombination Space]
\label{prop:recombination}
Consider a privacy bucket $\mathcal{B}_k$ containing $n_k$ clients, where each locally private update has $d$ parameters. Under parameter-level shuffling, reconstructing a coherent client update requires selecting one of the $n_k$ candidate values for each parameter index, yielding $n_k^d$ possible reconstructions. Recovering the full parameter-to-client assignment for the bucket results in $(n_k!)^d$ possible assignments.
\end{proposition}

\vspace{-0.5em}
\noindent Proposition~\ref{prop:recombination} shows that the reconstruction space grows exponentially with both bucket size and model dimension, fundamentally limiting the feasibility of recombination attacks. This theoretical result is corroborated empirically in Section~\ref{sec:recomb}.

\begin{theorem}[Preservation of Local DP]
\label{thm:ldp_post}
Suppose each client $i$ applies an $(\varepsilon_i,\delta_i)$ locally differentially private mechanism to produce $\tilde g_i^{(t)}$. Then the full IntraShuffler pipeline, including group formation, adaptive bucket merging, parameter-level shuffling, and subsequent server-side processing, preserves the same $(\varepsilon_i,\delta_i)$ privacy guarantees for each client.
\end{theorem}
\vspace{-0.3em}

\noindent \textit{Proof.}
All operations performed by IntraShuffler depend only on the outputs of the client-side local DP mechanisms. By the post-processing property of DP, deterministic or randomized transformations applied after the local DP mechanism cannot weaken the original guarantees.

\begin{figure}[t]
\centering
\begin{subfigure}{0.30\columnwidth}
    \centering
    \includegraphics[width=\linewidth]{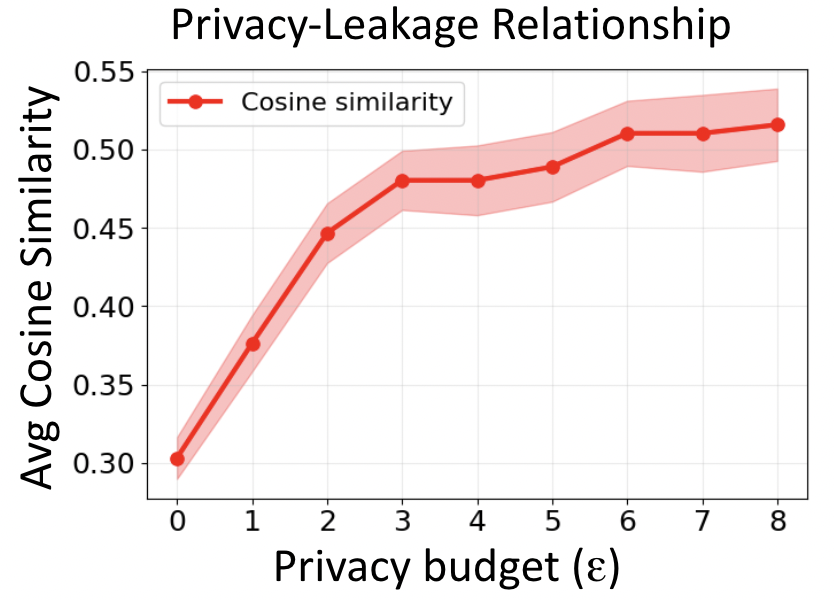}
    \caption{London Household}
    \label{fig:clean}
\end{subfigure}
\hspace{-1.5mm}
\begin{subfigure}{0.299\columnwidth}
    \centering
    \includegraphics[width=\linewidth]{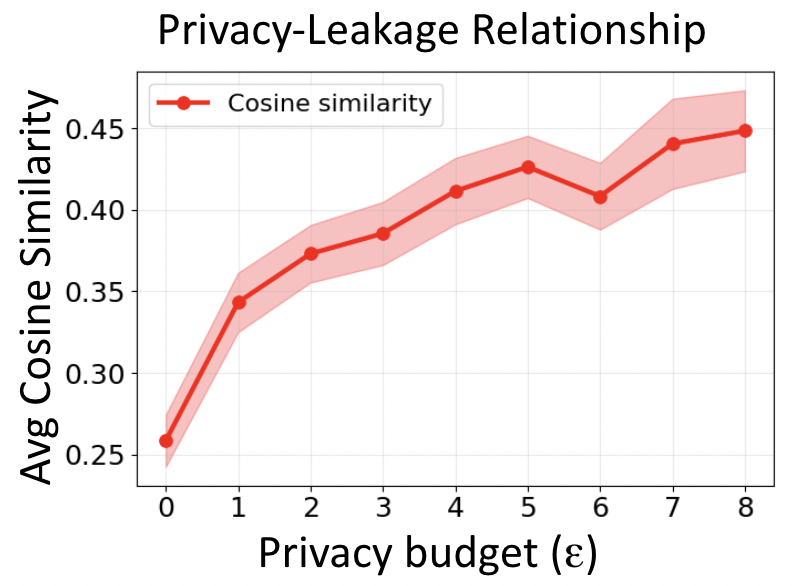}
    \caption{Pecan Street}
    \label{fig:ldp}
\end{subfigure}
\hspace{-1.5mm}
\begin{subfigure}{0.28\columnwidth}
    \centering
    \includegraphics[width=\linewidth]{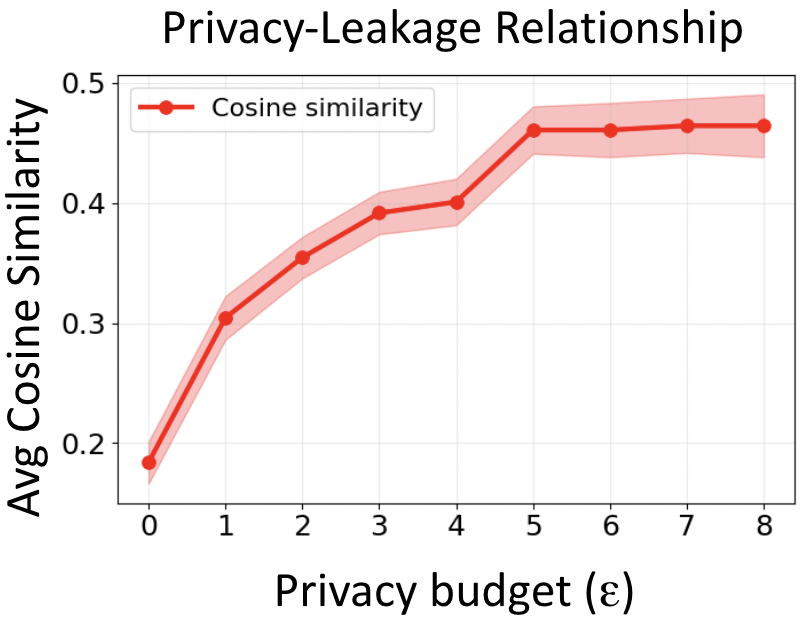}
    \caption{ComStock}
    \label{fig:denoised}
\end{subfigure}
\vspace{-2mm}
\caption{Comparison of denoiser performance in terms of average cosine similarity.}
\vspace{+0.3em}
\label{fig:cosine_similarity_comp}
% \vspace{-3mm}
\end{figure}

Fig.~\ref{fig:Overview} summarizes the IntraShuffler workflow. 
\step{1} Clients train local models on private data and \step{2} release locally differentially private updates using their privacy budgets $\varepsilon_i$. \step{3} Updates are sent to IntraShuffler, which groups them into privacy-compatible buckets. \step{4} Within each bucket, parameter-level shuffling permutes gradient components across clients, disrupting client-specific structure while preserving aggregate statistics. \step{5} The shuffled updates are forwarded to the server for $\varepsilon$-aware aggregation and global model updates. Inference is evaluated under a server equipped with denoisers and surrogate models.
% Fig.~\ref{fig:Overview} summarizes the IntraShuffler workflow.
% \step{1} Each client trains a local copy of the global model on private data and \step{2} releases a locally differentially private update using its chosen privacy budget $\varepsilon_i$. \step{3} The privatized updates are sent to IntraShuffler, which partitions them into buckets according to their declared privacy budgets (and corresponding noise scales). \step{4} Within each bucket, IntraShuffler performs parameter-level shuffling by independently permuting gradient components across clients, disrupting client-specific structure while preserving aggregate statistics. \step{5} The shuffled updates are then sent to the server for $\varepsilon$-aware aggregation and global model update, limiting gradient-structure leakage even under an honest-but-curious server. Inference is evaluated assuming a server equipped with denoisers and surrogate models

\begin{figure}[t]
\centering
\includegraphics[height=0.20\textheight]
{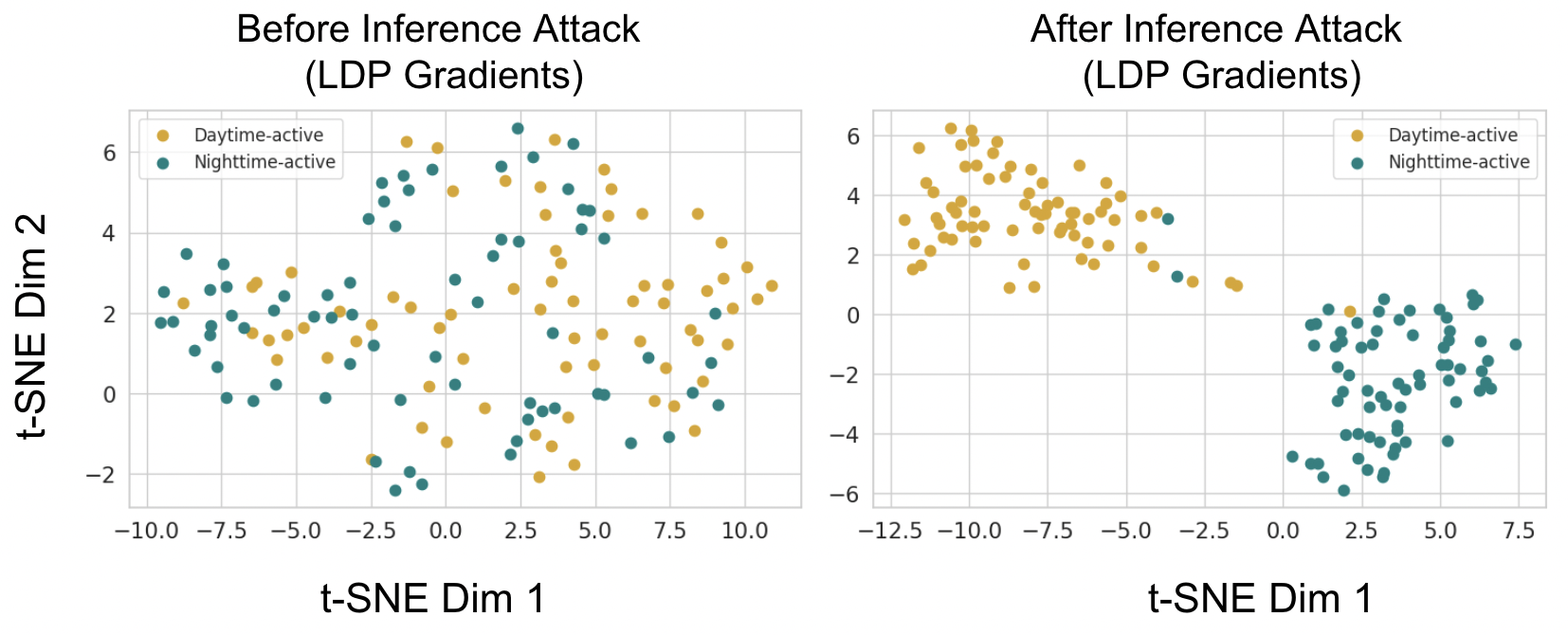}
\caption{Privacy leakage via $\varepsilon$-aware denoising. Server-side denoising of LDP gradients enables a surrogate model to partially infer client behavioral groups.}
\label{fig:seperation_withoutintraShuffler}
\vspace{+0.3em}
\end{figure}
\section{Evaluation}
\vspace{-0.6em}
\subsection{Experimental Setup}
\vspace{-0.6em}
All experiments were conducted with TPU-v2 acceleration. Each session provided approximately 51\,GB of system memory and 226\,GB of local disk storage.

\noindent\textbf{Datasets.}  
We evaluate IntraShuffler on real-world household-level load forecasting datasets, including the London Household Electricity and Pecan Street Electricity datasets, where each client corresponds to a household. These datasets capture behavioral patterns that may induce privacy leakage through distributional inference. We also include the large-scale ComStock dataset and evaluate cross-domain generality on the CIFAR-10 vision dataset.

\noindent\textbf{FL Setting.} 
We consider a standard FL setting with 200 clients, where 20 are randomly selected per round over 50 rounds. Each client performs $E=2$ local epochs and applies client-specific LDP using clipping norm $C$, $\delta=10^{-5}$, and Gaussian noise calibrated to $\varepsilon \in \{0.5,1,2,4,8\}$ under DP-SGD, with $\varepsilon$ treated as the per-round privacy parameter. Updates pass through IntraShuffler before aggregation, and results are averaged over multiple seeds. We use MobileNet and DenseNet for image classification, and LSTM and TimesFM for time-series forecasting, where the lightweight LSTM isolates the effect of aggregation and shuffling. Additional details are in Appendix~\ref{sec:appendix_utility}.

\noindent\textbf{Adversarial Evaluation.}
Under the threat model in Section~\ref{sec:threat}, we evaluate privacy leakage using two inference tools, a learned denoiser and a surrogate model trained offline on a disjoint historical dataset. Synthetic gradients with varying DP noise levels train an $\varepsilon$-conditioned denoiser, while the surrogate model learns domain-level gradient representations without client identifiers. During evaluation, both models operate on shuffled updates observable by the server to perform client linkage and distributional inference.
% \noindent\textbf{Adversarial Evaluation.}
% Under the threat model in Section~\ref{sec:threat}, we evaluate privacy leakage using a learned denoiser and a surrogate model trained offline on a disjoint historical dataset. Synthetic gradients with varying DP noise levels are used to train an $\varepsilon$-conditioned denoiser, while the surrogate learns domain-level gradient representations without client identifiers. During evaluation, both models operate on shuffled updates observed by the server to perform client linkage and distributional inference.

\noindent\textbf{Baseline Comparison.} We compare against Shuffle-DP and plain FL-DP as they represent the two dominant classes of defenses,anonymity-based and noise-based, while noting that, to our knowledge, no prior work studies parameter-level shuffling in HDP-FL under a shuffle-model architecture.
% Additional details are provided in Appendix~\ref{app:attackerdetails}.

\begin{wraptable}{r}{0.56\columnwidth}
\centering
\footnotesize
\setlength{\tabcolsep}{1pt}
\begin{tabular}{l|c|c|c}
\toprule
Method & CosSim & LinkAcc & Lin.Sep \\
\midrule
No Denoising & $0.12 \pm 0.02$ & $0.18 \pm 0.03$ & $0.20 \pm 0.02$ \\
$\varepsilon$-Denoising & $0.45 \pm 0.04$ & $0.76 \pm 0.03$ & $0.68 \pm 0.03$ \\
\bottomrule
\end{tabular}
\caption{Impact of denoiser on privacy leakage}
\label{tab:denoise_effect}
\end{wraptable}
\vspace{-0.5em}
\subsection{Gradient Leakage and Inference Risk}
\vspace{-0.5em}

We evaluate privacy leakage arising from server-side analysis of locally private updates and assess whether IntraShuffler mitigates the inference pipeline described in Section~\ref{sec:attack}. Specifically, we examine
gradient structure recovery via $\varepsilon$-aware denoising and downstream behavioral inference using surrogate models.
% \begin{wraptable}{r}{0.36\columnwidth}
% \centering
% \footnotesize
% \setlength{\tabcolsep}{0.5pt}
% \begin{tabular}{lcc}
% \toprule
% Method & CosSim |& LinkAcc \\
% \midrule
% No Denoising & $\approx0.12$ & $\approx0.18$ \\
% $\varepsilon$-Denoising & $\approx0.45$ & $\approx0.76$ \\
% \bottomrule
% \end{tabular}
% \caption{Impact of denoising}
% \label{tab:denoise_effect}
% \end{wraptable}
\begin{figure}[t]
    \centering
    \includegraphics[height=0.21\textheight]{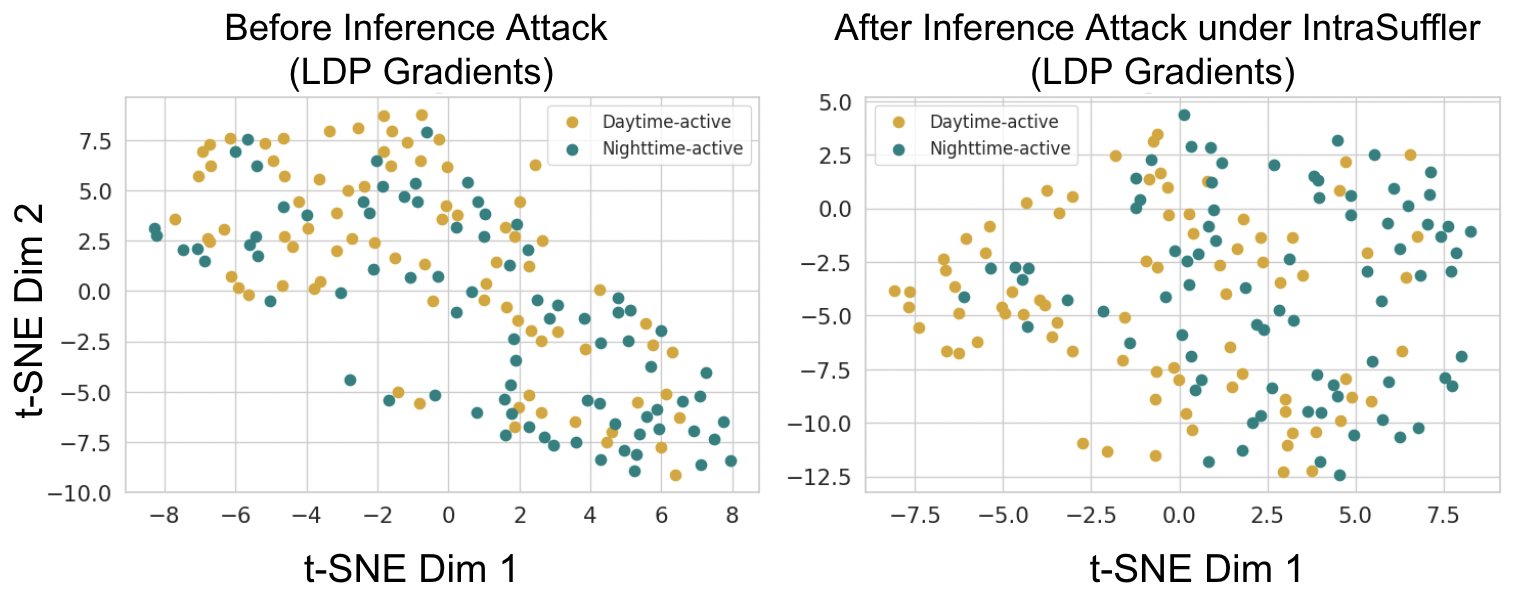}
    \caption{Effect of IntraShuffler on data distribution inference. Client representations remain blended, preventing the separation of similar groups.}
    \label{fig:seperation_withintraShuffler}
\end{figure}
\vspace{+0.5em}

\noindent\textbf{Leakage via $\varepsilon$-Aware Denoising.}
Without IntraShuffler, the server observes noisy gradients along with their privacy budgets and applies an $\varepsilon$-aware denoiser using the known noise scale. Although the recovered gradients remain imperfect, we observe consistent alignment between the denoised and true gradients, with cosine similarity typically between $0.35$ and $0.45$ across clients (Fig.~\ref{fig:cosine_similarity_comp}), revealing coarse structural information sufficient for inference.
\begin{wraptable}{r}{0.69\columnwidth}
\centering
\footnotesize
\setlength{\tabcolsep}{1pt}
\begin{tabular}{l|c|c|c|c}
\toprule
Method & CosAUC & $\Delta$Cos & Sur.Acc. & ARI \\
\midrule
No Shuff+DP & $0.42\pm0.02$ & $0.21\pm0.01$ & $0.78\pm0.03$ & $0.41\pm0.02$ \\
Shuffler-DP  & $0.36\pm0.02$ & $0.18\pm0.01$ & $0.69\pm0.03$ & $0.30\pm0.02$ \\
IntraShuffler   & $0.15\pm0.01$ & $0.05\pm0.01$ & $0.33\pm0.02$ & $0.06\pm0.01$ \\
\bottomrule
\end{tabular}
\caption{Comparison of privacy leakage under defenses.}
\label{tab:leakage_summary}
\end{wraptable}
Table~\ref{tab:denoise_effect} shows the role of denoising for reliable linkage, where denoising increases cosine similarity from $0.12$ to $0.45$, improves linkage accuracy, and raises linear separability (logistic regression) from near-random $0.20$ to $0.68$. Table~\ref{tab:leakage_summary} summarizes leakage across multiple metrics, cosine recovery (Cos AUC, the area under the cosine-similarity recovery curve across privacy budgets), sensitivity to privacy budgets ($\Delta$Cos), surrogate inference accuracy (Sur.\ Acc.), and clustering separability (ARI), where higher values indicate stronger leakage. Without IntraShuffler, gradients remain partially recoverable and behaviorally separable, whereas IntraShuffler substantially reduces all leakage metrics.
\begin{wrapfigure}{r}{0.42\columnwidth}
    \centering
    \vspace{-1pt}
    \includegraphics[width=\linewidth, height=3.7cm]{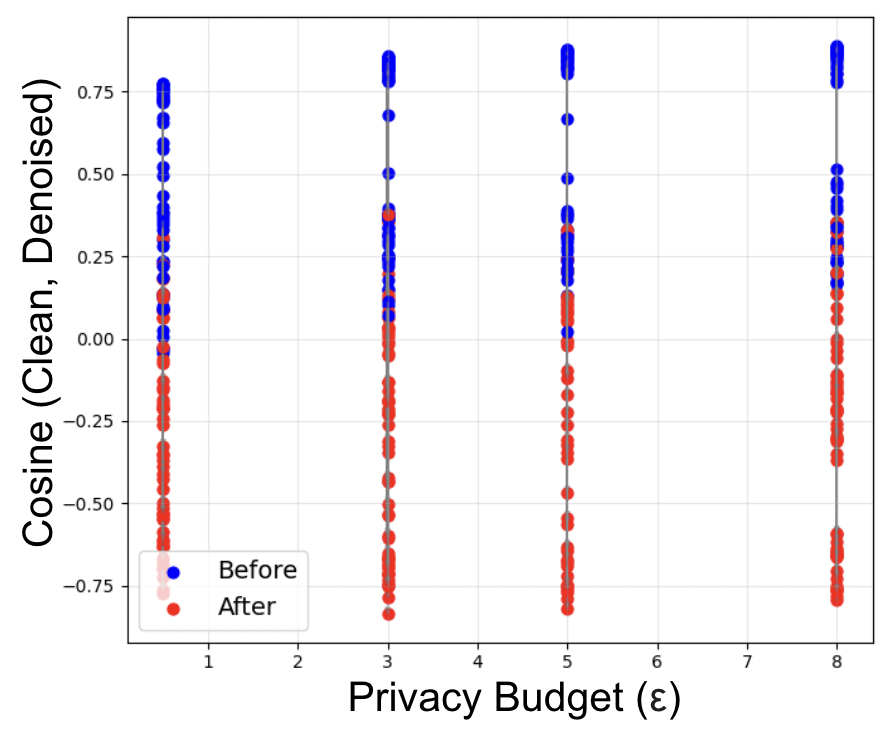}
    \caption{Cosine similarity comparison of denoised gradients before and after IntraShuffler}
    \label{fig:wrap_intrashuffler}
    \vspace{-1pt}
\end{wrapfigure}

\noindent\textbf{Behavioral Inference.} We next evaluate if this recovered structure enables inference about client data distributions. The server trains a surrogate model on a disjoint historical dataset and uses the denoised gradients as input features. In the load forecasting datasets, this allows the server to distinguish behavioral groups such as daytime-active versus nighttime-active households. These behavioral labels are derived from aggregate consumption statistics in the historical dataset and are not taken from federated clients. As shown in Fig.~\ref{fig:seperation_withoutintraShuffler}, raw LDP gradients exhibit substantial overlap and do not directly reveal client behavior; however, $\varepsilon$-aware denoising partially restores latent gradient structure, enabling the surrogate model to separate the two behavioral groups.

\noindent\textbf{Mitigation via IntraShuffler.} Applying IntraShuffler disrupts the structural consistency exploited by the attack. By shuffling gradient parameters within privacy-compatible buckets, the mechanism breaks the correspondence between gradient components and individual clients while preserving aggregate statistics for server aggregation. As shown in Fig.~\ref{fig:gradientLeak2}, $\varepsilon$-aware denoising fails to recover meaningful gradient structure when IntraShuffler is applied, resulting in substantially lower cosine similarity between denoised and true gradients (Fig.~\ref{fig:wrap_intrashuffler}). Consequently, the surrogate model is no longer able to reliably separate behavioral groups
(Fig.~\ref{fig:seperation_withintraShuffler}), indicating that both gradient reconstruction and downstream behavioral inference are significantly reduced. Across datasets, IntraShuffler reduces cosine similarity by more than 60\% compared to the non-shuffled setting, indicating substantial suppression of recoverable gradient structure.

\vspace{+0.3em}
\begin{figure}[htbp]
    \centering
    \begin{subfigure}{0.75\columnwidth}
        \centering
        \includegraphics[width=\columnwidth,height=3cm]{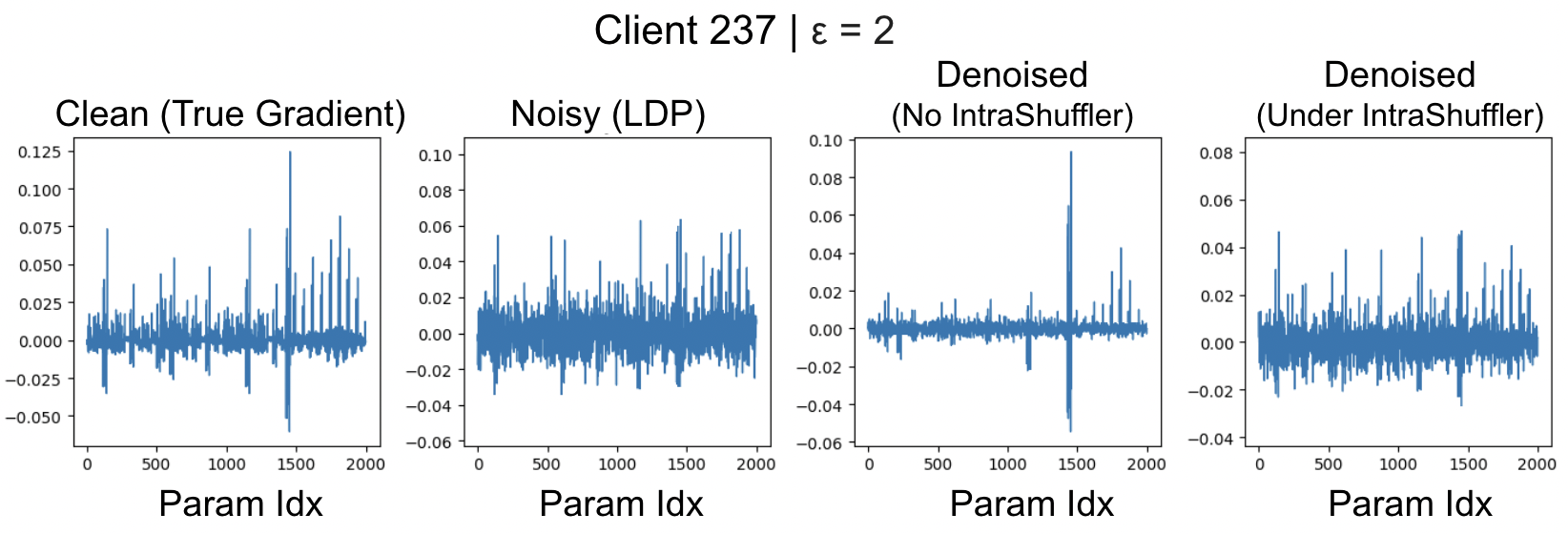}
        % \caption{First}
        \label{fig:a}
    \end{subfigure}
    \hfill
    \begin{subfigure}{0.75\columnwidth}
        \centering
        \includegraphics[width=\columnwidth,height=3cm]{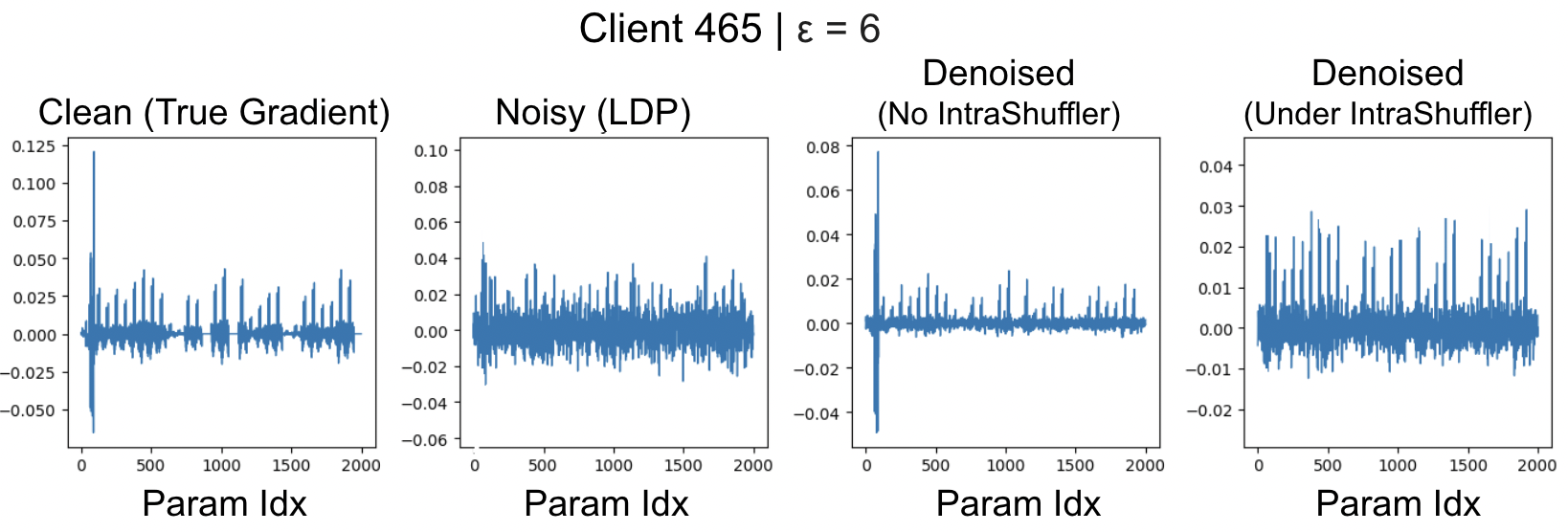}
        % \caption{Second}
        \label{fig:b}
    \end{subfigure}
    \caption{Effect of IntraShuffler on gradient recovery. Without shuffling, denoising partially restores gradients; IntraShuffler disrupts alignment and prevents recovery even at higher $\epsilon$.}
    \label{fig:gradientLeak2}
\end{figure}
\vspace{+0.7em}

% We evaluate cross-round linkability by testing whether a server-side adversary can match updates from the same client across rounds. Following Section~\ref{sec:attack}, the server computes cosine similarity between denoised gradients and performs nearest-neighbor matching. We report source inference accuracy under four settings, \emph{No Shuffle, No DP}, \emph{No Shuffle with DP}, \emph{Shuffler-DP}, and \emph{IntraShuffler}, where higher accuracy indicates stronger leakage. Figure~\ref{fig:round_likability} shows that without shuffling or DP, source inference is nearly perfect, confirming strong client-specific signatures. LDP reduces linkability only marginally, as residual structure persists, while message-level shuffling further lowers accuracy but still allows moderate linkage across rounds. In contrast, IntraShuffler consistently achieves near-random source inference accuracy across datasets and rounds. By performing parameter-level shuffling within privacy-compatible buckets, the mechanism breaks the structural consistency of gradients generated by the same client over time. As a result, updates become statistically indistinguishable across clients, preventing reliable cross-round linkage. These results demonstrate that while LDP and conventional shuffling offer only partial mitigation, IntraShuffler effectively eliminates long-term client linkability in  HDP-based FL.
\vspace{-1.0em}
\subsection{Cross-Round Linkability Analysis}
\vspace{-0.5em}
We evaluate cross-round linkability by testing whether a server-side adversary can match updates from the same client across rounds. Following Section~\ref{sec:attack}, the server computes cosine similarity between denoised gradients and applies nearest-neighbor matching. We report source
inference accuracy under four settings: \emph{No Shuffle, No DP}, \emph{No Shuffle with DP}, \emph{Shuffler-DP}, and \emph{IntraShuffler}, where higher accuracy indicates stronger leakage. Figure~\ref{fig:round_likability} shows that source inference is nearly perfect without shuffling or DP, confirming strong client-specific signatures. LDP provides only marginal protection, and message-level shuffling further reduces but does not eliminate linkability. In contrast, IntraShuffler achieves near-random source inference accuracy across datasets and rounds by shuffling parameters within privacy-compatible buckets, breaking the structural consistency of same-client gradients over time. These results show that while LDP and conventional shuffling provide partial mitigation, IntraShuffler effectively eliminates long-term client linkability in HDP-based FL.

\begin{table*}[t]
\centering
\footnotesize
\setlength{\tabcolsep}{7pt}
\renewcommand{\arraystretch}{1.1}
\begin{tabular}{c|c|ccc}
\hline
\textbf{Model} & \textbf{Aggregator} 
& \textbf{Plain FL-DP} & \textbf{Shuffle-DP} & \textbf{IntraShuffler} \\
\hline
\multirow{3}{*}{LSTM}
 & FedAvg  & $0.442 \pm 0.016$ & $0.423 \pm 0.014$ & $0.424 \pm 0.014$ \\
 & FedProx & $0.439 \pm 0.015$ & $0.421 \pm 0.013$ & $0.422 \pm 0.013$ \\
 & FedOpt  & $0.433 \pm 0.014$ & $0.417 \pm 0.012$ & $0.418 \pm 0.012$ \\
\hline
\multirow{3}{*}{TimesFM}
 & FedAvg  & $0.409 \pm 0.015$ & $0.392 \pm 0.013$ & $0.393 \pm 0.013$ \\
 & FedProx & $0.406 \pm 0.014$ & $0.390 \pm 0.012$ & $0.391 \pm 0.012$ \\
 & FedOpt  & $0.401 \pm 0.013$ & $0.386 \pm 0.011$ & $0.387 \pm 0.011$ \\
\hline
\end{tabular}
\vspace{+0.5em}
\caption{Final validation RMSE (mean $\pm$ std) on the ComStock dataset under different aggregation rules, privacy mechanisms, and global models.}
\label{tab:Utility}
\end{table*}

\vspace{-0.5em}
\subsection{Utility and Convergence Analysis}
\vspace{-0.3em}

\noindent We evaluate the impact of IntraShuffler on model utility and convergence under heterogeneous privacy budgets, comparing it with non-shuffled FL and Shuffle-DP. Table~\ref{tab:Utility} reports final RMSE on ComStock across models and aggregation rules. IntraShuffler achieves utility comparable to Shuffle-DP, with differences within run-to-run variability, indicating that parameter-level shuffling preserves convergence and learning performance. This holds across FedAvg, FedOpt, and FedProx, as IntraShuffler operates at the server-side aggregation interface without altering client optimization. In contrast, Plain FL-DP exhibits higher error due to stronger perturbation. Additional results are in Appendix~\ref{sec:appendix_utility}.

\begin{wraptable}{r}{0.42\columnwidth}
\vspace{+0.3em}
\centering
\footnotesize
\setlength{\tabcolsep}{1pt}
\begin{tabular}{lccc}
\toprule
Method & London & Pecan & ComStock \\
\midrule
No Bucket & 0.577 & 0.551 & 0.518 \\
Bucket & \textbf{0.408} & \textbf{0.424} & \textbf{0.393} \\
\bottomrule
\end{tabular}
\vspace{-0.5em}
\caption{Impact of bucketing in IntraShuffler.}
\label{tab:importance_of_bucketing}
% \vspace{-0.5mm}
\end{wraptable}

\noindent\textbf{Role of Bucketing in Utility.} Table~\ref{tab:importance_of_bucketing} evaluates the role of bucketing in
IntraShuffler across three datasets using TimesFM. Without bucketing, global parameter-level shuffling mixes gradients with heterogeneous noise scales, diluting the contribution of low-noise updates under $\varepsilon$-aware aggregation. In our implementation, the server
re-weights updates according to their declared privacy budgets (Eq.~\ref{eq:eps-agg}), assigning larger weights to less noisy high-$\varepsilon$ updates. Privacy-aware bucketing preserves this re-weighting while still enabling parameter-level shuffling, resulting in substantially lower forecasting error.

\vspace{-0.7em}
\subsection{Bucket-Level Privacy Analysis}
\vspace{-0.5em}

HDP introduces privacy imbalance because clients with larger $\varepsilon$ contribute less noisy gradients and therefore expose more recoverable structure. To quantify this effect, we evaluate the attack success rate separately for each privacy bucket. The attacker applies $\varepsilon$-aware denoising followed by surrogate behavioral inference,
and we report the resulting inference accuracy per bucket. 
\begin{wraptable}{r}{0.52\columnwidth}
\centering
\footnotesize
\setlength{\tabcolsep}{1pt}
\begin{tabular}{lc|c|c}
\toprule
Bucket & NoShuff+DP & Shuff-DP & Ours \\
\midrule
$\varepsilon\le1$  & $0.41\pm0.02$ & $0.38\pm0.02$ & $0.32\pm0.01$ \\
$1<\varepsilon\le4$ & $0.63\pm0.04$ & $0.55\pm0.03$ & $0.34\pm0.02$ \\
$\varepsilon>4$ & $0.81\pm0.03$ & $0.72\pm0.03$ & $0.36\pm0.02$ \\
\bottomrule
\end{tabular}
\vspace{-0.5em}
\caption{Bucket-level attack success.}
\label{tab:bucket_attack}
\end{wraptable}
Table~\ref{tab:bucket_attack} shows that attack success increases with larger privacy budgets in the absence of IntraShuffler, confirming that clients in high-$\varepsilon$ buckets are significantly more vulnerable. While message-level shuffling partially reduces leakage, the attack remains effective because gradient structure persists within buckets. IntraShuffler eliminates this imbalance by disrupting parameter-level structure within each bucket, reducing attack accuracy to near-random levels across all privacy groups.

\vspace{-0.5em}
\subsection{Feasibility of Recombination Attacks}
\label{sec:recomb}
\vspace{-0.5em}

We examine whether parameter recombination attacks remain feasible under our threat model. Prior work~\cite{Athanasiou_2024} shows that last-layer recombination can enable source inference when the attacker has prior knowledge of the target client’s data distribution, which provides a reference for validating candidate reconstructions. Without such knowledge, the server has no reliable criterion to identify the correct permutation. In our setting, the server must rely on downstream inference (denoising and surrogate modeling (Sec.~\ref{sec:attack})) to evaluate each candidate, introducing an additional step per recombination. At the same time, reconstructing a coherent update from a bucket of size $n_k$ and model dimension $d$ requires searching over $n_k^d$ possible combinations (Table~\ref{tab:recomb_space}), which grows exponentially and makes exhaustive search infeasible. Runtime also increases rapidly with bucket size, making it even more impractical. These results indicate that, under client-side DP and parameter-level shuffling, recombination-based SIAs are both computationally and statistically ineffective.

\begin{wraptable}{r}{0.42\columnwidth}
\centering
\scriptsize
\setlength{\tabcolsep}{2pt}
\begin{tabular}{c|c|c|c}
\toprule
$n_k$ & $d$ & Search Space & 1 eval/sec \\
\midrule
4  & $10^4$ & $4^{10^4} \approx 10^{6021}$ & $10^{6021}$ sec \\
8  & $10^4$ & $8^{10^4} \approx 10^{9030}$ & $10^{9030}$ sec \\
16 & $10^4$ & $16^{10^4} \approx 10^{12041}$ & $10^{12041}$ sec \\
\bottomrule
\end{tabular}
\vspace{-0.5em}
\caption{Combinatorial growth of recombination space and runtime.}
\label{tab:recomb_space}
\end{wraptable}

% \vspace{+0.5em}
\begin{figure}[t]
    \centering
    \begin{subfigure}{0.24\columnwidth}
        \centering
        \includegraphics[width=\linewidth]{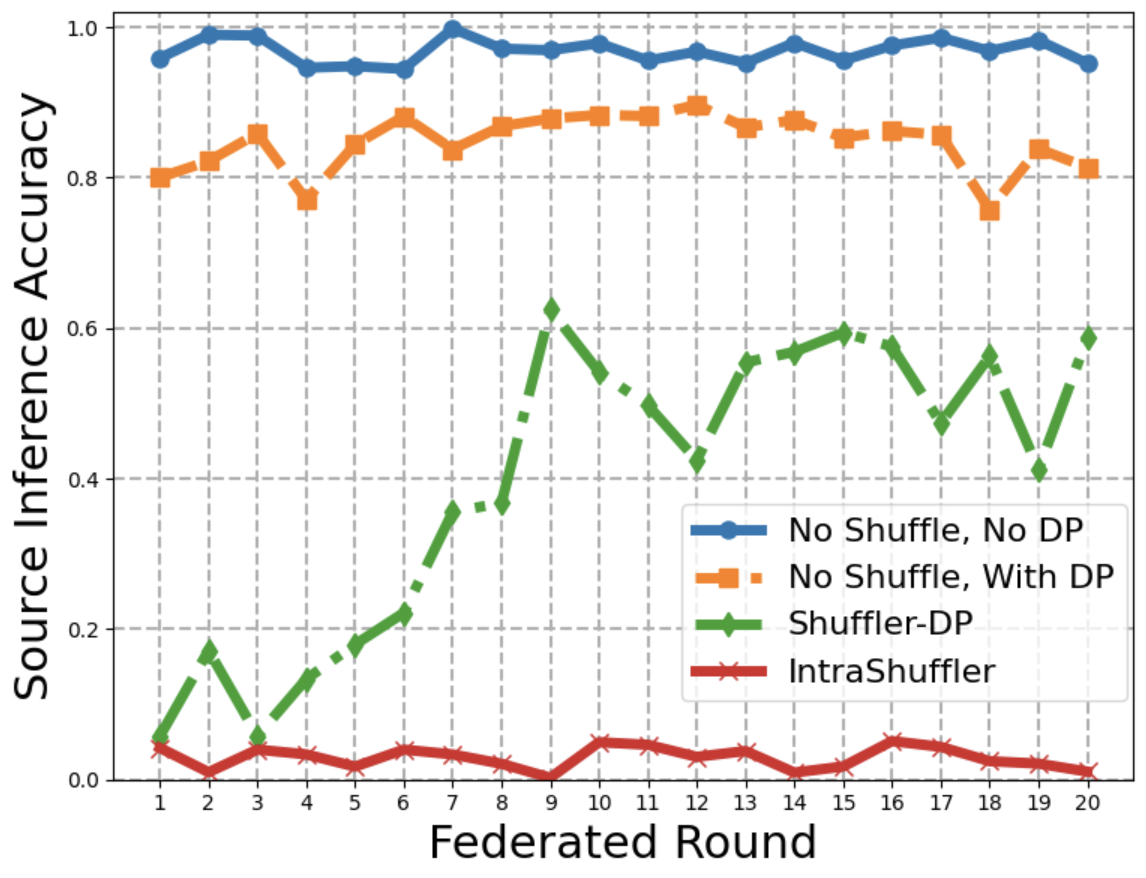}
        \caption{}
        \label{fig:a}
    \end{subfigure}\hfill
    \begin{subfigure}{0.24\columnwidth}
        \centering
        \includegraphics[width=\linewidth]{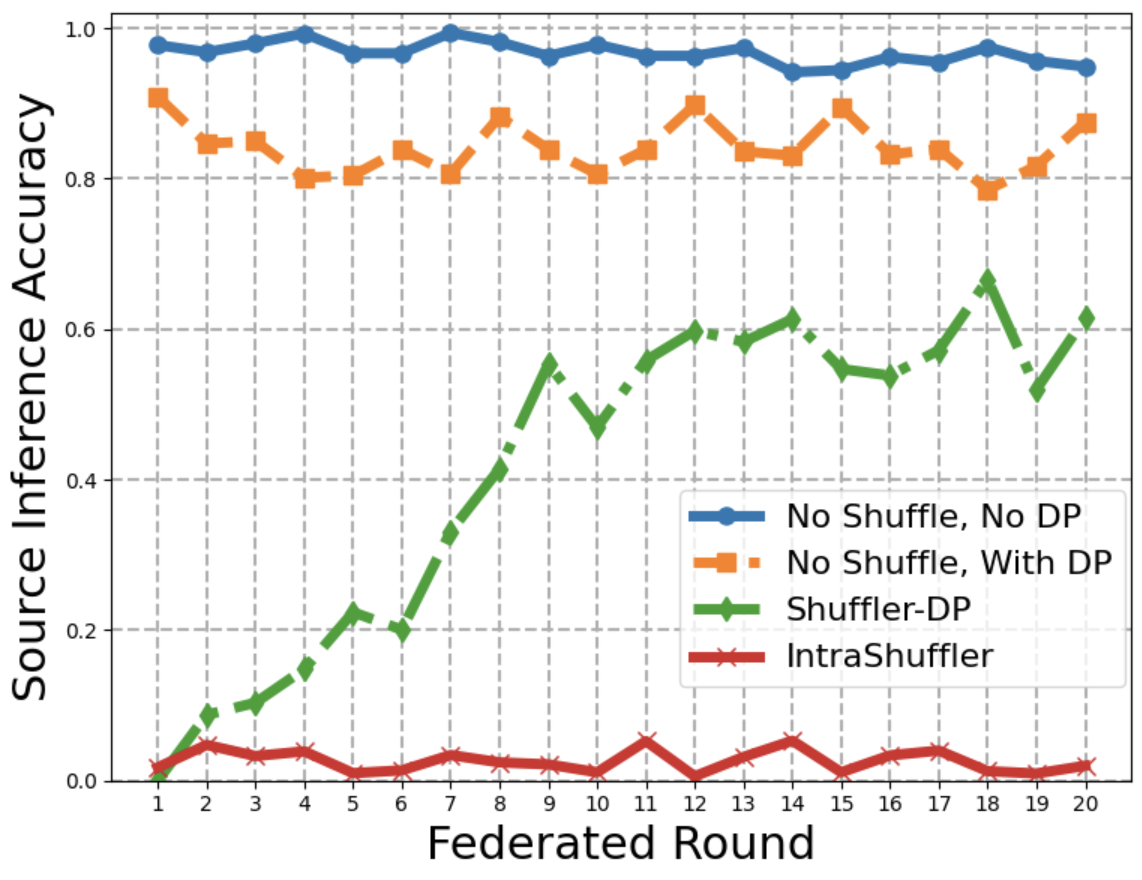}
        \caption{}
        \label{fig:b}
    \end{subfigure}\hfill
    \begin{subfigure}{0.24\columnwidth}
        \centering
        \includegraphics[width=\linewidth]{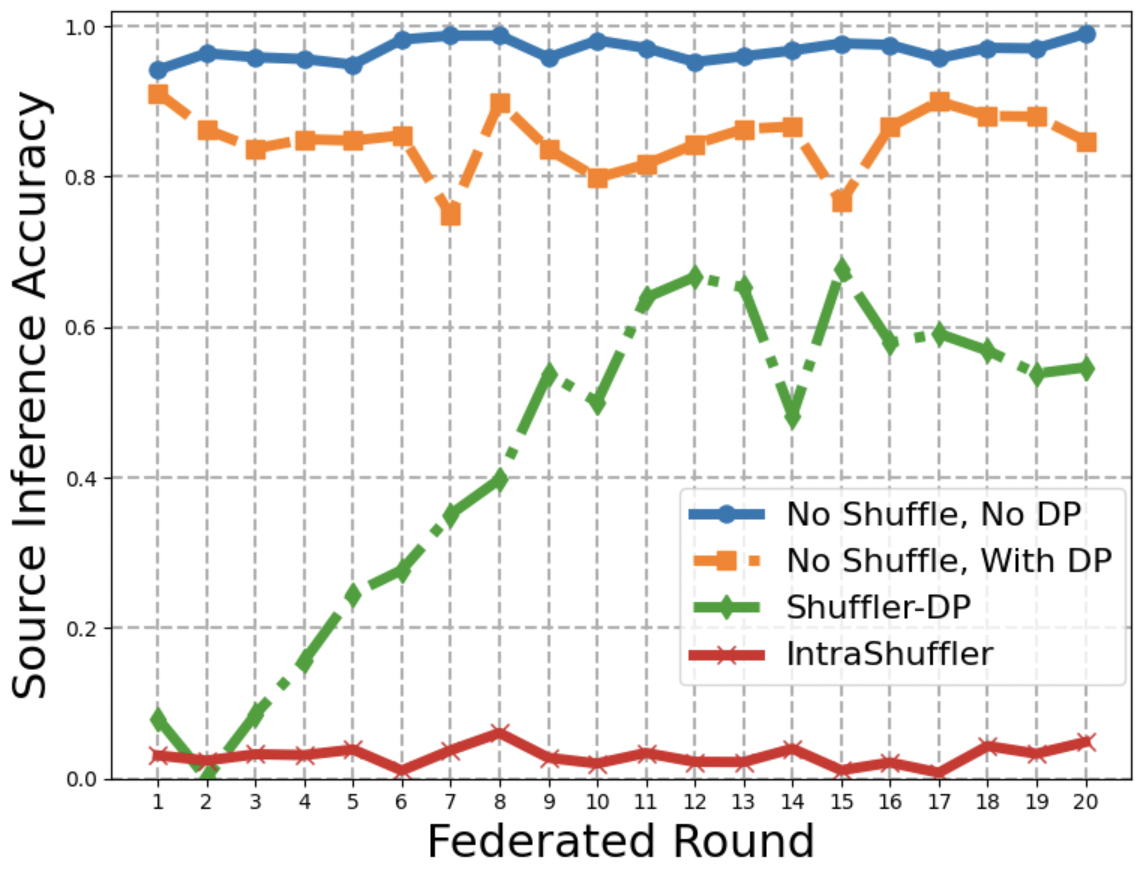}
        \caption{}
        \label{fig:c}
    \end{subfigure}\hfill
    \begin{subfigure}{0.24\columnwidth}
        \centering
        \includegraphics[width=\linewidth]{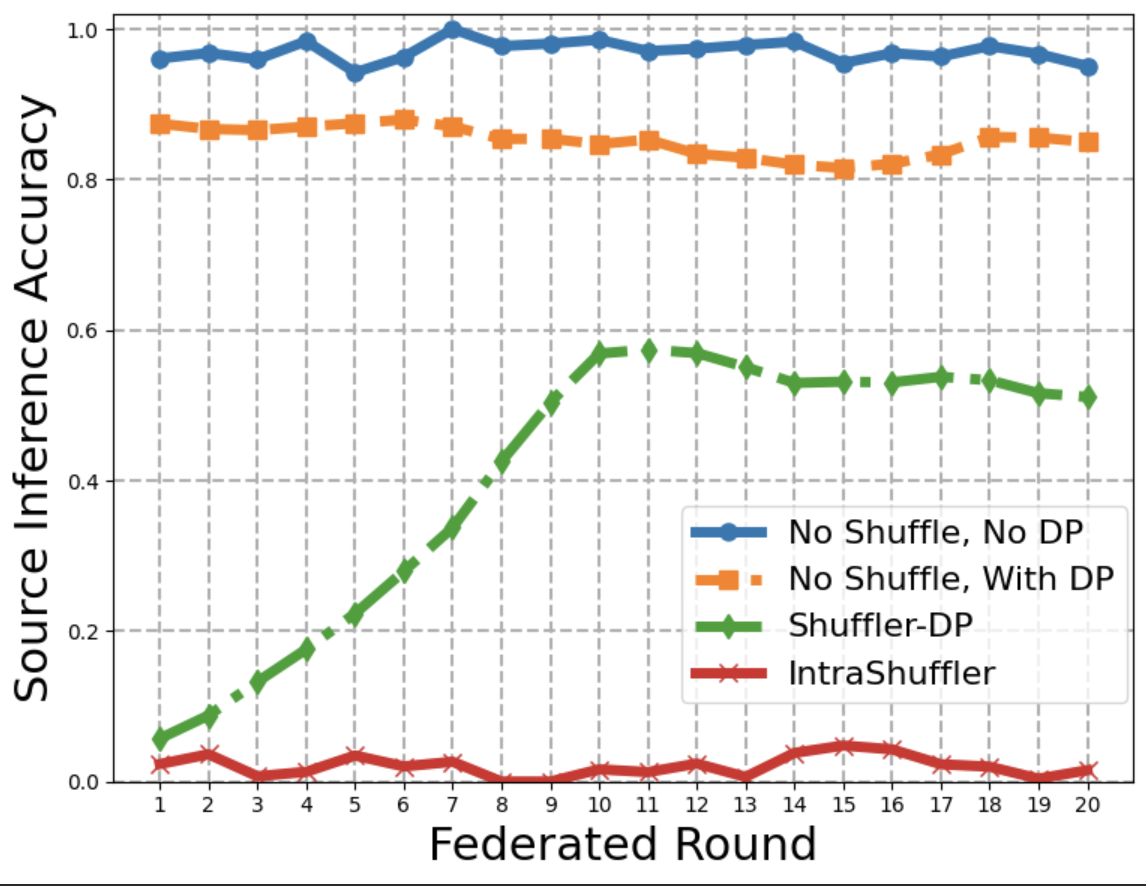}
        \caption{}
        \label{fig:d}
    \end{subfigure}
    \caption{Linkability analysis of a client participating in multiple rounds of FL (a) London Household, (b) Pecan Street, (c) ComStock, and (d) CIFAR-10.}
    \vspace{+0.5em}
    \label{fig:round_likability}
\end{figure}

% asymptotic storage cost.

\vspace{-0.5em}
\section{Communication Efficiency and IntraShuffler}
\vspace{-0.7em}

IntraShuffler introduces additional server-side computation due to privacy-aware bucketing and intra-bucket parameter-level shuffling, but does not change the communication pattern or message size between clients and the server. Its goal is to mitigate privacy leakage under HDP and non-IID data, rather than to reduce communication overhead. Let $m$ denote the number of participating clients, $L$ the number of layers, and $d_\ell$ the number of parameters in layer $\ell$. Bucketing incurs $O(m)$ time. For each bucket $\mathcal{B}_k$, parameter-level shuffling costs \(O\!\left(\sum_{\ell=1}^{L} |\mathcal{B}_k| d_\ell\right)\), yielding a total per-round cost of \(O\!\left(\sum_{\ell=1}^{L} m d_\ell\right)=O(mD)\), where $D=\sum_{\ell=1}^{L} d_\ell$. This is linear in the model size and comparable to standard server-side aggregation. Table~\ref{tab:comm_overhead1} summarizes the asymptotic overhead. The memory overhead is $O(mD)$, as the shuffler stores all privatized updates for a round, matching standard aggregation without additional asymptotic cost. While larger models (e.g., TimesFM) incur higher absolute cost than smaller ones (e.g., LSTM), the asymptotic complexity remains unchanged.

\vspace{+0.5em}
\begin{table}[htbp]
\centering
\footnotesize
\setlength{\tabcolsep}{5pt}
\renewcommand{\arraystretch}{1.15}
\begin{tabular}{l|c|c|c|c}
\hline
\textbf{Component} & \textbf{Depends on} & \textbf{Cost} & \textbf{LSTM} & \textbf{TimesFM} \\
\hline
Privacy bucketing & $m$ & $O(m)$ & negligible & negligible \\
parameter-level shuffling & $m, L, D$ &
$O(mD)$ &
$O(mD_{\text{LSTM}})$ &
$O(mD_{\text{TFM}})$ \\
Aggregation (FedAvg) & $m, D$ &
$O(mD)$ &
$O(mD_{\text{LSTM}})$ &
$O(mD_{\text{TFM}})$ \\
\hline
\end{tabular}
\vspace{+0.3em}
\caption{Server-side overhead of IntraShuffler as a function of the number of clients $m$, model depth $L$, and total parameter count $D=\sum_{\ell=1}^{L} d_\ell$.}
\label{tab:comm_overhead1}
\end{table}

\vspace{-1em}
\section{Limitations and Future Directions}
\vspace{-1.0em}
A fundamental limitation of bucket-based anonymization approaches arises when the number of participating clients is small and privacy budgets are highly diverse, leading to low bucket populations. In the worst case, where $n$ clients have $n$ distinct privacy budgets, each bucket contains a single client, eliminating the effectiveness of parameter-level shuffling entirely. Additionally, IntraShuffler focuses on privacy heterogeneity, while practical FL systems often exhibit multiple forms of heterogeneity, including data distribution, system capabilities, and participation dynamics. A natural direction is to design adaptive bucketing strategies that jointly account for multiple sources of heterogeneity, enabling more robust anonymization while balancing privacy and utility in realistic federated settings. Another direction is to explore compatibility with complementary techniques such as secure aggregation, trusted execution environments, and communication-efficient compression.

\vspace{-1em}
\section{Conclusion}
\vspace{-1.0em}
This work shows that in HDP-FL, privacy leakage arises not only from insufficient noise but from the persistence of gradient structure under $\varepsilon$-aware aggregation and non-IID data. We demonstrate that existing shuffle-based privacy amplification mechanisms, while effective in homogeneous settings, do not address this leakage and may destabilize aggregation when anonymity is maximized without regard to heterogeneity. By addressing this gap, we propose a defense that combines privacy-aware bucketing with parameter-level shuffling to disrupt gradient structure while preserving $\varepsilon$-aware aggregation. Experiments show that IntraShuffler significantly reduces gradient recoverability and inference accuracy while maintaining model utility across multiple FL optimizers.

\section{Acknowledgment}
\vspace{-0.8em}
% \vspace{-0.5em}
This material is based upon work supported by the U.S. Department of Energy, Office of Science, Office of Advanced Scientific Computing Research under Contract No. DE-AC05-00OR22725. This manuscript has been co-authored by UT-Battelle, LLC under Contract No. DE-AC05-00OR22725 with the U.S. Department of Energy. The United States Government retains and the publisher, by accepting the article for publication, acknowledges that the United States Government retains a non-exclusive, paid-up, irrevocable, world-wide license to publish or reproduce the published form of this manuscript, or allow others to do so, for United States Government purposes. The Department of Energy will provide public access to these results of federally sponsored research in accordance with the DOE Public Access Plan (http://energy.gov/downloads/doe-public-access-plan).

\bibliographystyle{IEEEtran}

\bibliography{ref}

\appendix

% \section*{Acknowledgment}
% \vspace{-0.8em}
% % \vspace{-0.5em}
% This material is based upon work supported by the U.S. Department of Energy, Office of Science, Office of Advanced Scientific Computing Research under Contract No. DE-AC05-00OR22725. This manuscript has been co-authored by UT-Battelle, LLC under Contract No. DE-AC05-00OR22725 with the U.S. Department of Energy. The United States Government retains and the publisher, by accepting the article for publication, acknowledges that the United States Government retains a non-exclusive, paid-up, irrevocable, world-wide license to publish or reproduce the published form of this manuscript, or allow others to do so, for United States Government purposes. The Department of Energy will provide public access to these results of federally sponsored research in accordance with the DOE Public Access Plan (http://energy.gov/downloads/doe-public-access-plan).

\vspace{-0.8em}
\section{Utility and Convergence Results}
\label{sec:appendix_utility}
\vspace{-0.8em}
This appendix reports additional utility results on load forecasting (London Household Electricity, Pecan Street) and image classification (CIFAR-10 with MobileNet and DenseNet). These complement the ComStock results and show that IntraShuffler preserves utility across datasets, models, and tasks.
% \vspace{-0.5em}

\noindent\textbf{A.1 Load Forecasting on Energy Datasets}
% \vspace{-0.5em}
We evaluate IntraShuffler on the London Household Electricity dataset and
the Pecan Street dataset using the same LSTM and TimesFM (zero-shot) forecasting setup and HDP as in the main experiments. Table~\ref{tab:appendix_load} reports the final validation RMSE under different aggregation rules and privacy mechanisms. Across both datasets, IntraShuffler consistently matches the utility of standard Shuffle-DP and significantly outperforms Plain FL-DP, indicating that parameter-level shuffling does not degrade convergence or forecasting accuracy.

\vspace{+0.3em}
\begin{table*}[htbp]
\centering
\scriptsize
\setlength{\tabcolsep}{3pt}
\renewcommand{\arraystretch}{1.0}
\begin{tabular}{c|c|c|ccc}
\hline
\textbf{Dataset} & \textbf{Model} & \textbf{Agg.}
& \textbf{Plain FL-DP} & \textbf{Shuffle-DP} & \textbf{IntraShuffler} \\
\hline
\multirow{6}{*}{London Household}
 & \multirow{3}{*}{LSTM}
 & FedAvg  & $0.468 \pm 0.018$ & $0.445 \pm 0.015$ & $0.446 \pm 0.015$ \\
 &  & FedProx & $0.462 \pm 0.017$ & $0.440 \pm 0.014$ & $0.441 \pm 0.014$ \\
 &  & FedOpt  & $0.455 \pm 0.016$ & $0.433 \pm 0.013$ & $0.434 \pm 0.013$ \\
\cline{2-6}
 & \multirow{3}{*}{TimesFM}
 & FedAvg  & $0.451 \pm 0.017$ & $0.428 \pm 0.014$ & $0.429 \pm 0.014$ \\
 &  & FedProx & $0.446 \pm 0.016$ & $0.424 \pm 0.013$ & $0.425 \pm 0.013$ \\
 &  & FedOpt  & $0.439 \pm 0.015$ & $0.418 \pm 0.012$ & $0.419 \pm 0.012$ \\
\hline
\multirow{6}{*}{Pecan Street}
 & \multirow{3}{*}{LSTM}
 & FedAvg  & $0.421 \pm 0.016$ & $0.402 \pm 0.014$ & $0.403 \pm 0.014$ \\
 &  & FedProx & $0.417 \pm 0.015$ & $0.398 \pm 0.013$ & $0.399 \pm 0.013$ \\
 &  & FedOpt  & $0.410 \pm 0.014$ & $0.392 \pm 0.012$ & $0.393 \pm 0.012$ \\
\cline{2-6}
 & \multirow{3}{*}{TimesFM}
 & FedAvg  & $0.404 \pm 0.015$ & $0.386 \pm 0.013$ & $0.387 \pm 0.013$ \\
 &  & FedProx & $0.400 \pm 0.014$ & $0.382 \pm 0.012$ & $0.383 \pm 0.012$ \\
 &  & FedOpt  & $0.394 \pm 0.013$ & $0.377 \pm 0.011$ & $0.378 \pm 0.011$ \\
\hline
\end{tabular}
\vspace{+0.3em}
\caption{RMSE for load forecasting
under HDP on energy datasets, comparing LSTM and TimesFM
models across aggregation rules and privacy mechanisms.}
\label{tab:appendix_load}
\end{table*}

% \vspace{-1.0em}
\noindent\textbf{A.2 Image Classification on CIFAR-10:}
% \vspace{-0.5em}
We evaluate IntraShuffler on CIFAR-10 using MobileNet and DenseNet under HDP. Table~\ref{tab:appendix_cifar} reports test accuracy. Consistent with load forecasting results, IntraShuffler preserves accuracy across architectures and aggregation rules, closely matching Shuffle-DP while outperforming Plain FL-DP. Overall, it maintains utility and convergence without degradation, generalizing across both time-series and image classification tasks.

\begin{table*}[t]
\centering
\scriptsize
\setlength{\tabcolsep}{8pt}
\renewcommand{\arraystretch}{1.1}
\begin{tabular}{c|c|ccc}
\hline
\textbf{Model} & \textbf{Aggregator}
& \textbf{Plain FL-DP} & \textbf{Shuffle-DP} & \textbf{IntraShuffler} \\
\hline
\multirow{3}{*}{MobileNet}
 & FedAvg  & $71.8\% \pm 1.2$ & $74.6\% \pm 1.0$ & $74.4\% \pm 1.0$ \\
 & FedProx & $72.3\% \pm 1.1$ & $75.1\% \pm 0.9$ & $75.0\% \pm 0.9$ \\
 & FedOpt  & $73.5\% \pm 1.0$ & $76.2\% \pm 0.8$ & $76.0\% \pm 0.8$ \\
\hline
\multirow{3}{*}{DenseNet}
 & FedAvg  & $78.4\% \pm 1.1$ & $81.6\% \pm 0.9$ & $81.4\% \pm 0.9$ \\
 & FedProx & $78.9\% \pm 1.0$ & $82.1\% \pm 0.8$ & $82.0\% \pm 0.8$ \\
 & FedOpt  & $80.1\% \pm 0.9$ & $83.3\% \pm 0.7$ & $83.1\% \pm 0.7$ \\
\hline
\end{tabular}
\vspace{+0.3em}
\caption{CIFAR-10 test accuracy (mean $\pm$ std) under heterogeneous
DP for different architectures and aggregation rules.}
\label{tab:appendix_cifar}
\end{table*}

% \vspace{-0.8em}
\noindent\textbf{A.3 Granularity of Shuffling:}
% \vspace{-0.5em}
We compare parameter-level shuffling (IntraShuffler) with a coarser layer-wise baseline on the London Household Electricity dataset using an LSTM model under HDP. Layer-wise shuffling permutes entire layers across clients within a bucket while preserving intra-layer parameter structure. As shown in Table~\ref{tab:appendix_granularity}, parameter-level shuffling significantly reduces gradient alignment and inference accuracy, indicating that disrupting fine-grained structure is critical for mitigating inference attacks.

\vspace{-0.7em}
\section{Malicious Shuffler}
\vspace{-1.0em}
\label{app:malicious_shuffler}

IntraShuffler follows the standard shuffle-model assumption of an independent, non-colluding intermediary~\cite{bittau2017prochlo,cheu2022differentialprivacyshufflemodel}. The shuffler processes only LDP-protected updates and their bucket identifiers, and has no access to model architecture, training hyperparameters, client data distributions, or auxiliary/historical data. As a result, even if independently malicious, it lacks the side information needed to exploit gradient structure for inference beyond the guarantees of the local randomizer. If the shuffler colludes with the server, the anonymity benefit of shuffling is removed. The system then reduces to the baseline LDP setting: client-level privacy guarantees remain unchanged by post-processing, but no additional unlinkability is provided. This behavior is consistent with prior shuffle-model analyses, where collusion eliminates shuffle-based anonymity and reverts protection to local DP~\cite{bittau2017prochlo,wang2020improving}.

\begin{wraptable}{r}{0.45\columnwidth}
\centering
\scriptsize
\setlength{\tabcolsep}{1pt}
\renewcommand{\arraystretch}{1}
\begin{tabular}{c|c|c}
\toprule
\textbf{Shuffling} & \textbf{CosSim} $\downarrow$ & \textbf{Inf. Acc.} $\downarrow$ \\
\midrule
Layer-wise & $0.41 \pm 0.03$ & $0.62 \pm 0.04$ \\
Parameter-level & $0.12 \pm 0.02$ & $0.28 \pm 0.03$ \\
\bottomrule
\end{tabular}
\vspace{-0.4em}
\caption{Granularity comparison on London Household (LSTM).}
\label{tab:appendix_granularity}
% \vspace{-0.8em}
\end{wraptable}

\vspace{-0.7em}
\section{IID vs Non-IID Ablation}
\vspace{-1.0em}
\label{app:non_iid}
To assess the role of client heterogeneity in gradient-based inference, we compare IID and non-IID partitions on CIFAR-10. Both settings use the same number of clients and samples per client. In the \emph{IID} case, clients receive balanced class mixtures, while in the \emph{non-IID} case, clients exhibit class-skew with imbalanced label distributions.

% \vspace{+0.3em}
\begin{wraptable}{r}{0.48\columnwidth}
\centering
\scriptsize
\setlength{\tabcolsep}{2pt}
\renewcommand{\arraystretch}{1}
\begin{tabular}{c|c|c|c}
\toprule
\textbf{Partition} & \textbf{Plain} & \textbf{Shuffle} & \textbf{Intra} \\
\midrule
IID     & $\approx 1/n_t$ & $\approx 1/n_t$ & $\approx 1/n_t$ \\
Non-IID & $0.42$ & $0.31$ & $0.12$ \\
\bottomrule
\end{tabular}
\vspace{-0.4em}
\caption{IID vs non-IID Cross-round linkage. Random baseline is $1/n_t$.}
\label{tab:iid_non_iid2}
% \vspace{-0.8em}
\end{wraptable}

Non-IID data induces client-specific gradient directions, increasing separability and enabling distributional inference and cross-round linkage. In contrast, IID partitions produce more homogeneous gradients, reducing inference signals. All settings share identical privacy budgets, DP parameters, and training configuration. Leakage is minimal under IID but increases significantly under non-IID due to stronger client-dependent structure. IntraShuffler consistently reduces linkage, approaching random performance in IID and substantially improving over Shuffle-DP in non-IID. This highlights the necessity of disrupting gradient structure in realistic non-IID FL settings.

\end{document}